\documentclass[12pt, final, onecolumn, notitlepage]{article}

\usepackage[backend=biber,style=apa,natbib=true]{biblatex}
\usepackage{amsmath, amsfonts, amssymb, amsthm, graphicx, enumitem, titlesec, comment, hyperref, bbm, algorithm, algcompatible, xcolor, thm-restate, authblk}

\allowdisplaybreaks
\titlelabel{\thetitle.\quad}
\usepackage[margin=1in]{geometry}
\usepackage[noend]{algpseudocode}

\newcommand{\vect}[1]{\mathbf{#1}}

\newcommand{\RR}{\mathbb{R}}

\newcommand{\GG}{\mathbb{G}}
\newcommand{\PP}{\mathbb{P}}
\newcommand\norm[1]{\left\lVert#1\right\rVert}
\newcommand\abs[1]{\left\lvert#1\right\rvert}

\newcommand\parens[1]{\left(#1\right)}

\newcommand\angles[1]{\langle#1\rangle}
\newcommand*{\pr}[2][]{\mathrm{Pr}\ifx\\\left[#1\right]\\\else_{#1}\fi \left[#2\right]}
\newcommand*{\EE}[2][]{\mathbb{E}\ifx\\\left[#1\right]\\\else_{#1}\fi \left[#2\right]}
\newcommand*{\Var}[2][]{\mathrm{Var}\ifx\\\left[#1\right]\\\else_{#1}\fi \left[#2\right]}
\newcommand*{\Cov}[2][]{\mathrm{Cov}\ifx\\\left(#1\right)\\\else_{#1}\fi \left(#2\right)}

\newcommand{\perm}{\mathrm{perm}}
\newcommand{\haf}{\mathrm{haf}}

\newcommand{\diag}{\mathrm{\mathbf{diag}}}
\newcommand{\row}{\mathrm{row}}
\newcommand{\col}{\mathrm{col}}
\newcommand{\ms}{\mathrm{ms}}

\newcommand{\tr}{\mathrm{Tr}}
\newcommand{\size}{\mathrm{rsize}}

\algdef{SE}[DOWHILE]{Do}{doWhile}{\algorithmicdo}[1]{\algorithmicwhile\ #1}%
\MakeRobust{\Call}

\newtheorem{theorem}{Theorem}[section]
\newtheorem{lemma}[theorem]{Lemma}
\newtheorem{prop}[theorem]{Proposition}

\newtheorem{claim}[theorem]{Claim}
\newtheorem{conjecture}[theorem]{Conjecture}

\theoremstyle{definition}
\newtheorem{defin}[theorem]{Definition}
\newtheorem{example}[theorem]{Example}
\newtheorem{remark}[theorem]{Remark}

\title{Towards a Law of Iterated Expectations for Heuristic Estimators}
\date{\today}
\author[ ]{Paul Christiano\footnote{Work done while at the Alignment Research Center prior to April 2024}}
\author[1]{Jacob Hilton}
\author[2]{Andrea Lincoln}
\author[1]{Eric Neyman\footnote{Corresponding author: \texttt{eric@alignment.org}}}
\author[1]{Mark Xu\footnote{Authors listed in alphabetical order}}
\affil[1]{Alignment Research Center}
\affil[2]{Boston University}

\begin{document}
\maketitle

\begin{abstract}
    \citet{christiano2022formalizing} define a \emph{heuristic estimator} to be a hypothetical algorithm that estimates the values of mathematical expressions from arguments. In brief, a heuristic estimator $\GG$ takes as input a mathematical expression $Y$ and a formal ``heuristic argument'' $\pi$, and outputs an estimate $\GG(Y \mid \pi)$ of $Y$. In this work, we argue for the informal principle that a heuristic estimator ought not to be able to predict its own errors, and we explore approaches to formalizing this principle. Most simply, the principle suggests that $\GG(Y - \GG(Y \mid \pi) \mid \pi)$ ought to equal zero for all $Y$ and $\pi$. We argue that an ideal heuristic estimator ought to satisfy two stronger properties in this vein, which we term \emph{iterated estimation} (by analogy to the law of iterated expectations) and \emph{error orthogonality.}

    Although iterated estimation and error orthogonality are intuitively appealing, it can be difficult to determine whether a given heuristic estimator satisfies the properties. As an alternative approach, we explore \emph{accuracy}: a property that (roughly) states that $\GG$ has zero average error over a distribution of mathematical expressions. However, in the context of two estimation problems, we demonstrate barriers to creating an accurate heuristic estimator. We finish by discussing challenges and potential paths forward for finding a heuristic estimator that accords with our intuitive understanding of how such an estimator ought to behave, as well as the potential applications of heuristic estimators to understanding the behavior of neural networks.
\end{abstract}

\section{Introduction} \label{sec:intro}
It is often possible to estimate a mathematical expression with a high degree of confidence, even when proving tight bounds on the expression is difficult or impossible. For example, let $Y$ be the number of primes between $e^N$ and $e^N + 1000 N$, where $N$ is 1 billion. The prime number theorem tells us that among integers of this size, roughly one in a billion are prime. Thus, it would be reasonable to estimate that $Y \approx 1000$, and we can be quite confident that $500 \le Y \le 2000$, barring some yet-undiscovered regularity in the pattern of prime numbers. By contrast, formally proving that $500 \le Y \le 2000$ may not be possible without a computationally intractable brute-force verification.

Arguments like this one -- known as \emph{heuristic arguments} -- analyze the structure of a problem in order to estimate a quantity. The estimate produced by a heuristic argument reflects a ``best guess'' about the quantity after taking into account some considerations. Correspondingly, the estimate is uncertain -- the argument is not a proof -- and could be revised in light of additional considerations. Heuristic arguments are common in fields such as number theory, discrete math, and theoretical computer science (see \citet{eu71, mz02, ch19} for examples).

Despite the ubiquity of heuristic arguments in mathematics, there has been little prior work attempting to formalize this style of reasoning. To our knowledge, the first in-depth attempt at formalization was \citet{christiano2022formalizing}. The authors introduced the notion of a heuristic estimation algorithm (henceforth \emph{heuristic estimator}), which takes as input a formally specified real-valued expression $Y$ together with a set of formal ``arguments'' $\pi_1, \dots, \pi_m$, and estimates the value of $Y$ by incorporating the information provided by $\pi_1, \dots, \pi_m$. The authors suggested that a heuristic estimator should be guided by a \emph{presumption of independence:} presuming two sub-parts of the expression $Y$ can be treated independently unless an argument points out a relationship between them.

\citet{christiano2022formalizing} discuss properties that a heuristic estimator ought to satisfy, such as linearity and respect for proofs (see Section~\ref{sec:new_properties} below). In this work, we will suggest and study a new property, inspired by the law of iterated expectations from probability theory: a heuristic estimator should not be able to easily predict its own errors.

\subsection{A running example} \label{sec:example}
In this work, we will use $\GG$ (for ``guesser'') to denote a hypothetical heuristic estimator, which takes as input a formal mathematical expression $Y$ and a set of arguments $\Pi = \{\pi_1, \dots, \pi_m\}$, and outputs an estimate of the value of $Y$ based on the arguments in $\Pi$. As we discuss below, $\Pi$ can be thought of as a ``state of knowledge'': a formal description of all facts considered by the heuristic estimator. Under this view, $\GG(Y \mid \Pi)$ is an estimate of $Y$ in light of these facts. We will use the notation $\GG(Y \mid \Pi)$ to denote this estimate.\footnote{\citet{christiano2022formalizing} use the notation $\tilde{\mathbb{E}}(Y, \pi_1, \dots, \pi_m)$. Our change in notation is meant to emphasize that, while there are similarities between heuristic estimation and expected values, they are importantly different.} In the case of $\Pi = \{\pi\}$, we will simply write $\GG(Y \mid \pi)$.

In this section, we introduce a simple example in order to illustrate how we would like a heuristic estimator $\GG$ to behave.

Given a positive irrational number $x$ and a positive integer $k$, let $d_k(x)$ be the $k$-th digit of $x$ past the decimal point, when written in base 10. (For example, since $\sqrt{2} = 1.41\dots$, we have $d_1(\sqrt{2}) = 4$, $d_2(\sqrt{2}) = 1$, and so on.) For our example, we will take $Y$ to be the expression
\[Y := \sum_{n = 101}^{120} d_6(\sqrt{n}).\]
That is $Y$, is the sum of the sixth digits past the decimal point of $\sqrt{101}$, $\sqrt{102}$, and so on, up to $\sqrt{120}$.

As we will discuss in Section~\ref{sec:perspectives}, we would like $\GG$ to return a ``subjective'' expectation of $Y$ in light of the arguments provided as input. So first, suppose that $\GG$ receives no arguments. In the absence of further information, it would be reasonable to believe that the sixth digit of $\sqrt{n}$ is equally likely to be each of $0, 1, \dots, 9$. Thus, $\GG$'s subjective expectation of summand $d_6(\sqrt{n})$ should be the average of $0, 1, \dots, 9$, i.e.\ $4.5$. Thus, we would like $\GG$ to satisfy\footnote{More precisely, instead of being given no arguments, $\GG$ would likely be given a short argument that points out that there are twenty summands and its estimate for each summand ought to be $4.5$.}
\[\GG(Y \mid \emptyset) = 20 \cdot 4.5 = 90.\]

Now, let $\pi_n$ be a proof of the value of $d_6(\sqrt{n})$. For example, $\pi_{101}$ proves that $d_6(\sqrt{101}) = 5$, perhaps by showing that $10.04987\color{red}5\color{black}^2 < 101 < 10.04987\color{red}6\color{black}^2$.

In light of $\pi_{101}$, $\GG$ should change its belief about $d_6(\sqrt{101})$ to $5$, but should not change its beliefs about $d_6(\sqrt{n})$ for any other $n$. Thus, we would like
\begin{equation} \label{eq:pi_101}
    \GG(Y \mid \pi_{101}) = 5 + 19 \cdot 4.5 = 90.5.
\end{equation}

It may be useful to imagine $\GG$ being presented with $\pi_{101}$, $\pi_{102}$, and so on in sequence. After observing each additional argument, $\GG$ takes the argument into account in order to refine its estimate. At the end of the process, $\GG$'s estimate is exactly correct:
\begin{align*}
    \GG(Y \mid \emptyset) &= 20 \cdot 4.5 = 90\\
    \GG(Y \mid \pi_{101}) &= 5 + 19 \cdot 4.5 = 90.5\\
    \GG(Y \mid \pi_{101}, \pi_{102}) &= 5 + 4 + 18 \cdot 4.5 = 90\\
    &\dots\\
    \GG(Y \mid \pi_{101}, \pi_{102}, \dots, \pi_{120}) &= 5 + 4 + 1 + 9 + \dots + 0 + 2 + 1 = 78
\end{align*}

In practice, the arguments given to $\GG$ can be significantly more complex than simple partial computations of $Y$. As a simple example, consider the argument $\pi_{101}'$ which shows that $10.04987\color{red}5\color{black}^2 < 101 < 10.04987\color{red}9\color{black}^2$, and thus that $d_7(\sqrt{101})$ is $5$, $6$, $7$, or $8$. Such an argument should cause $\GG$ to update its subjective belief about $d_7(\sqrt{101})$ from ``uniform over $\{0, 1, \dots, 9\}$'' to ``uniform over $\{5, 6, 7, 8\}$'' and to update its estimate of $Y$ accordingly. Of course, arguments can cause $\GG$ to update its estimates in much more complicated ways as well.

\subsection{Perspectives on heuristic estimation} \label{sec:perspectives}
In this section, we clarify the purpose and desired behavior of a heuristic estimator by analogizing heuristic estimates to three other concepts: proof verification, conditional expected values, and subjective probabilities and estimates.

\paragraph{Analogy to proof verification.} As discussed in \citet{christiano2022formalizing, neyman2024algorithmic}, a heuristic estimator can be thought of as a generalization of a formal proof verifier. A proof verifier takes as input a formal mathematical statement and a purported proof, and checks whether the proof proves the statement. Similarly, a heuristic estimator takes as input a formal mathematical expression and an argument (or arguments) about the expression, and outputs an estimate of the value of the expression that reflects those arguments. Importantly, the purpose of a heuristic estimator is to incorporate the heuristic arguments that it has been given as input, rather than to generate its own arguments. The output of a heuristic estimator is only as sophisticated as the arguments that it has been given.

As discussed by \citet{christiano2022formalizing}, $\GG$ would ideally respect proofs: roughly speaking, if $\pi$ contains a proof that $\ell \le Y \le h$, then $\ell \le \GG(Y \mid \pi) \le h$. This is the sense in which heuristic estimators \emph{generalize} proof verifiers, rather than just being analogous to them. Table~\ref{table:he_pv_analogy} illustrates the analogy to proof verification in more detail.

\begin{table}[ht]
    \centering
    \begin{tabular}{c||c}
        \textbf{Heuristic estimation} & \textbf{Proof verification}\\
        \hline
        Heuristic estimator & Proof verifier\\
        \hline
        Formal mathematical expression & Formal mathematical statement\\
        \hline
        List of heuristic arguments & Purported proof of statement\\
        \hline
        Formal language for heuristic arguments & Formal language for proofs\\
        \hline
        Desiderata for estimator & Soundness and completeness\\
        \hline
        Algorithm's estimate of expression & Verifier's output (accept or reject)\\
    \end{tabular}
    \caption{We are interested in developing a heuristic estimator for mathematical expressions. There are similarities between this task and the (solved) task of developing an algorithm for verifying formal proofs of mathematical statements. This table illustrates the analogy. (Adapted from \citet[Chapter 9]{neyman2024algorithmic}.)}
    \label{table:he_pv_analogy}
\end{table}

\paragraph{Analogy to conditional expectations.} In some ways, a heuristic estimator is analogous to a conditional expected value. In probability theory, the expectation of a random variable $X$ conditioned on an event $A$, $\EE{X \mid A}$ is the average value of $X$ over the subset of outcomes in which $A$ occurs. Intuitively, it is the estimate of $X$ given by an observer who does not know the exact outcome and instead only knows that $A$ occurred. Similarly, if $Y$ is a mathematical expression and $\Pi$ is a set of heuristic arguments, then $\GG(Y \mid \Pi)$ is an estimate of $Y$ given by an observer who has not computed the exact value of $Y$ and has instead only done the computations described in $\Pi$. Although there is a particular correct value of $Y$, the observer does not know this value, and $\GG(Y \mid \Pi)$ is a subjective ``best guess'' about $Y$ given only $\Pi$.

\paragraph{Analogy to subjective probabilities and estimates.} Perhaps most intuitively, $\GG$ is a procedure that extracts a subjective expectation from a state of knowledge. Under this view, $\Pi$ formally describes a set of facts known by an observer, and $\GG(Y \mid \Pi)$ is an estimate of $Y$ in light of those facts. To clarify this perspective, we recall the notion of subjective expectations.

The subjectivist view of probability interprets probability as the subjective credence of an observer. For example, suppose that a coin was created to have a bias (i.e.\ probability of heads) that was uniformly sampled from $[0, 1]$, and the coin was flipped. According to an observer who does not know the bias (but knows that the bias was sampled uniformly), the subjective probability that the coin came up heads is $0.5$. According to an observer who knows that the bias is $p$, the subjective probability of heads is $p$. According to an observer who saw the outcome of the coin flip, the probability is $1$ if the coin came up heads and $0$ if it came up tails. Indeed, orthodox subjectivists would argue that even if the bias were selected using an unknown procedure, an observer would still have a subjective probability that the coin came up heads; this probability is governed by the observer's priors about how coins' biases are determined.

Just as observers may have subjective probabilities of events, they may have subjective expected values of quantities. Recall our example from Section~\ref{sec:example}. A typical mathematician does not know $d_6(\sqrt{101})$ (the sixth digit of $\sqrt{101}$ past the decimal point), but their subjective probability distribution over the digit is uniform over $\{0, \dots, 9\}$. Correspondingly, their subjective expectation for $d_6(\sqrt{101})$ is $4.5$ (the average of $0, \dots, 9$). Similarly, the mathematician's subjective expectation for $Y := d_6(\sqrt{101}) + \dots + d_6(\sqrt{120})$ is $20 \cdot 4.5 = 90$. If the mathematician were to learn that $d_6(\sqrt{101}) = 5$, then they would update their subjective expectation to $5 + 19 \cdot 4.5 = 90.5$.

A heuristic estimator is meant to formalize the reasoning process undertaken by the mathematician, much as illustrated in Section~\ref{sec:example}. Under this view of heuristic estimation, $Y$ is some formal mathematical expression that an observer is uncertain about (perhaps a large summation, or perhaps a more complicated expression), $\Pi$ is the observer's knowledge, and $\GG$ returns a subjective estimate of $Y$ based on the knowledge encoded in $\Pi$.

\subsection{The principle of unpredictable errors} \label{sec:new_properties}
\textcite{christiano2022formalizing} suggest that in order for $\GG$ to be a coherent, general-purpose estimator, it ought to satisfy some formal properties. For example:
\begin{itemize}
    \item \textbf{Correctly estimating constants:} If the expression $Y$ is a hard-coded constant $c$, then $\GG$ should estimate $Y$ correctly. That is, for $c \in \RR$, for all $\Pi$, we have $\GG(c \mid \Pi) = c$. For example, if $\GG$ correctly estimates constants, then $\GG(2 \mid \Pi) = 2$ for all $\Pi$ (but $\GG(2 \cdot 2 \mid \Pi)$ is not necessarily $4$).
    \item \textbf{Linearity:} For $a, b \in \RR$ and expressions $X, Y$, $\GG$'s estimate of the expression $aX + bY$ is linear in its estimates of $X$ and $Y$ -- that is, for all $\Pi$, we have
    \[\GG(aX + bY \mid \Pi) = a \GG(X \mid \Pi) + b \GG(Y \mid \Pi).\]
    \item \textbf{Respect for proofs:} Given a proof that $Y \ge 0$, the proof may be turned into a heuristic argument $\pi$ such that $\GG(Y \mid \Pi) \ge 0$ for all $\Pi \ni \pi$.
\end{itemize}
Additionally, the authors suggest some informal properties -- such as \emph{presumption of independence} and \emph{independence of irrelevant arguments} -- though no formal statements are provided.\\

The basis of this work is the following informal principle: \emph{a heuristic estimator should not be able to predict its own errors.} We call this the \emph{principle of unpredictable errors.} The main technical content of this work concerns the formalization of this principle.

The principle of unpredictable errors is motivated by our earlier analogy of heuristic estimates to subjective Bayesian expectations. A Bayesian reasoner cannot predict the direction in which they will update their estimate in light of new information: if they could, then they would make that update before receiving the information. More formally, a Bayesian reasoner's subjective estimate of their own future subjective estimate of a given quantity should be equal to their current estimate of the quantity. This is known as the \emph{martingale property.}

We will discuss two approaches to formalizing the principle of unpredictable errors. The first approach (introduced in Section~\ref{sec:iterated_estimation}) -- which we call the \emph{subjective} approach -- involves two formal properties of $\GG$: \emph{iterated estimation} and \emph{error orthogonality}. Both properties are inspired by laws that govern conditional expected values. The \emph{iterated estimation} property states that for all expressions $Y$ and for all sets of arguments $\Pi$ and $\Pi' \subseteq \Pi$, we have
\begin{equation} \label{eq:iterated_est}
    \GG(\GG(Y \mid \Pi) \mid \Pi') = \GG(Y \mid \Pi').
\end{equation}
In other words, if $\GG$ is given a small set of arguments $\Pi'$, its guess about its estimate of $Y$ if it were to consider a larger set of arguments $\Pi$ should be its current estimate of $Y$. This law is inspired by the law of iterated expectations in probability theory.

(Why is this desirable? Consider our running example from Section~\ref{sec:example}. If given no arguments, $\GG$'s best estimate of $\GG(Y \mid \pi_{101})$ should be its current estimate of $90$: before considering $\pi_{101}$, $\GG$ should not be able to predict how $\pi_{101}$ will update its estimate. This is analogous to aforementioned martingale property of Bayesian reasoners.)

The \emph{error orthogonality} property is a more sophisticated version of the iterated estimation property (though not a strict generalization).\footnote{Error orthogonality only generalizes iterated estimation if certain additional assumptions are made about $\GG$: for example, one may need to assume that for all expressions $Z$ and sets of arguments $\Pi_1, \Pi_2$, we have $\GG(Z G(1 \mid \Pi_1) \mid \Pi_2) = G(Z \mid \Pi_2)$. Note that this property does not follow from linearity and correctly estimating constants.} Informally, error orthogonality states that the error of $\GG$'s estimate of $Y$ given $\Pi$ should not be predictable from $\GG$'s estimate of any other quantity (given $\Pi$ or a subset thereof). Formally, error orthogonality states that for all expressions $X, Y$ and for all sets of arguments $\Pi$ and $\Pi_1, \Pi_2 \subseteq \Pi$, we have
\begin{equation} \label{eq:error_orth}
    \GG((Y - \GG(Y \mid \Pi)) \cdot \GG(X \mid \Pi_1) \mid \Pi_2) = 0.
\end{equation}
In other words, the outer $\GG$ is tasked with estimating the subjective covariance\footnote{See Section~\ref{sec:iterated_estimation} for a discussion of what we mean by ``subjective covariance.''} between two quantities: the error in $\GG$'s estimate of $Y$ given a set of arguments $\Pi$, and $\GG$'s estimate of $X$ given a smaller set of arguments $\Pi_1$. The error orthogonality property states that this subjective covariance must be zero.

If the error orthogonality property were to not hold, it would mean that adding some constant multiple of $\GG(X \mid \Pi_1)$ to $\GG(Y \mid \Pi)$ would produce an improved estimate of $Y$. In other words: knowing $\GG(X \mid \Pi_1)$ makes the error of $\GG(Y \mid \Pi)$ predictable.

Our discussion of iterated estimation and error orthogonality will lead naturally to another approach to formalizing the principle of unpredictable errors, which we call the \emph{objective} approach (Section~\ref{sec:multiacc}). We define a new class of properties of $\GG$, which we call \emph{accuracy} properties. Concretely, iterated estimation closely resembles \emph{1-accuracy,} while error orthogonality closely resembles \emph{self-accuracy.} The key difference is that while iterated estimation and error orthogonality concern $\GG$'s estimates of its own outputs ($\GG$'s errors are \emph{subjectively} unpredictable to $\GG$), accuracy properties concern the expected value of $\GG$'s outputs over a known distribution ($\GG$'s errors are \emph{objectively} unpredictable over a distribution). We summarize these properties in Table~\ref{table:error_unpredictability}.

\begin{table}[ht]
    \centering
    \renewcommand{\arraystretch}{1.2}
    \textbf{\large Properties formalizing the principle of unpredictable errors} \\[0.5em]
    \begin{tabular}{c|c}
        \textbf{Subjective properties} & \textbf{Objective properties} \\
        \hline
        Iterated estimation: for $\Pi' \subseteq \Pi$, & 1-accuracy: \\
        $\GG(\GG(Y \mid \Pi) \mid \Pi') = \GG(Y \mid \Pi')$ & $\EE[Y \sim \mathcal{D}]{Y - \GG(Y \mid \Pi)} = 0$ \\ 
        \hline
        Error orthogonality: for $\Pi_1, \Pi_2 \subseteq \Pi$, & Self-accuracy: for $\Pi' \subseteq \Pi$, \\
        $\GG((Y - \GG(Y \mid \Pi)) \cdot \GG(X \mid \Pi_1) \mid \Pi_2) = 0$ & $\EE[Y \sim \mathcal{D}]{(Y - \GG(Y \mid \Pi')) \cdot \GG(Y \mid \Pi)} = 0$ \\
    \end{tabular}
    \caption{Summary of correspondences between subjective properties concerning the unpredictability of $\GG$'s errors and the corresponding objective properties. Note that 1-accuracy is a sub-case of the more general $X$-accuracy property (equation~\ref{eq:g_acc}).}
    \label{table:error_unpredictability}
\end{table}

\subsection{Outline of this work}
In Section~\ref{sec:iterated_estimation}, we motivate and discuss the subjective approach to formalizing the principle of unpredictable errors. We also discuss challenges with the approach. Concretely, the iterated estimation and error orthogonality properties are meant to help guide a search for a ``reasonable'' heuristic estimator; however, the properties involve $\GG$'s estimate of its own output, which poses two challenges. First, $\GG$ might be able to satisfy the iterated estimation property simply by treating input expressions of the form $\GG(Y \mid \Pi)$ as a special case, effectively satisfying the iterated estimation property ``by fiat.'' Second, reasoning about $\GG$'s estimate of its own output may be unwieldy, because $\GG$ is likely to be a nontrivial algorithm.

In light of these challenges, in Section~\ref{sec:multiacc} we take inspiration from the subjective approach, and introduce the objective approach, which -- loosely speaking -- requires that $\GG$ have unpredictable errors over a distribution of possible inputs. Concretely, given a distribution $\mathcal{D}$ over mathematical expressions $Y$ and a predictor $X$, we define an estimator $f(Y)$ to be \emph{$X$-accurate} with respect to $Y$ if $\EE[\mathcal{D}]{(Y - f(Y))X} = 0$: in other words, the error of $f$'s estimate of $Y$ is uncorrelated with $X$ over $\mathcal{D}$. We use this notion of accuracy to define accuracy properties of heuristic estimators over a distribution.

We explore accuracy in the context of two case studies. First, in Section~\ref{sec:hafnian}, we consider the problem of estimating the expected product of jointly normal random variables. We ask whether it is possible to efficiently compute an estimator of the expected product that is $X$-accurate for a small and simple collection of predictors $X$. We present a formal computational barrier: we show that computing an accurate \emph{linear} estimator is $\mathsf{\#P}$-hard. We further conjecture that computing \emph{any} accurate estimator is computationally intractable. We then define and discuss \emph{approximately} accurate estimation. We find that approximately accurate estimates are efficiently computable under certain assumptions about the predictors $X$, but leave the question open in the general case.

Our second case study (Section~\ref{sec:perm}) is estimating the permanent of a matrix. We consider three predictors for the permanent, and show that an accurate estimator with respect to the three predictors can be obtained via linear regression. However, this estimator fails certain sanity checks (it sometimes outputs negative estimates for the permanent of positive matrices). We find that, although there is a natural estimator in terms of the three predictors that \emph{does} pass this sanity check, it is not accurate with respect to the predictors. Together, the results of our two case studies lead us to conclude that the objective approach to formalizing the principle of unpredictable errors is unlikely to succeed.

In Section~\ref{sec:revisiting}, we put our discussion of accuracy in context by revisiting our desiderata for $\GG$. Our rejection of the objective approach leads us to pose the following question: is there a formalization of the principle of unpredictable errors that accords with our intuitive understanding of how heuristic estimators ought to behave?

Finally, in Section~\ref{sec:conclusion}, we discuss a potential application of heuristic estimation to understanding the behavior of neural networks. We also discuss the main challenges ahead for constructing a reasonable heuristic estimator and briefly describe a new perspective on heuristic explanation. In the context of neural networks, this perspective views heuristic arguments as distributional models of a neural network's internal activations.

\subsection{Related work}
This work builds on \citet{christiano2022formalizing}, which asked whether it is possible to formalize the heuristic arguments and heuristic estimation. The authors posited that the foundational principle for heuristic estimation is a \emph{presumption of independence:} two expressions should be presumed independent, unless there is an argument to the contrary. Our work is also closely related to \citet[Chapter 9]{neyman2024algorithmic}, which explores desiderata for heuristic arguments (particularly linearity and respect for proofs) in the context of boolean circuits. By contrast, our work focuses on exploring and developing a new desideratum that aims to capture the notion that a heuristic estimator ought not to be able to predict its own errors.

Besides this line of work, little prior work attempts to formalize heuristic estimation. Perhaps the closest is \citet{tao12}, which also explores the presumption of independence as the basis for heuristic argumentation in mathematics. The author uses the presumption of independence to heuristically justify the ABC conjecture from number theory.

Our work also bears resemblances to \citet{barak15, barak16}. The author explores whether there is an estimation algorithm whose estimate for a given quantity appears reasonable to a broad class of formal ``observers.'' These observers are analogous to heuristic arguments in our setting. However, while Barak explores estimation that appears reasonable to a large but pre-defined class of observers, we are interested in an estimation algorithm that appears reasonable to a particular set of observers (heuristic arguments) that are given as input. This narrower goal may allow the estimation algorithm to satisfy a much broader range of possible observers.

\citet{garrabrant16induction} explore the assignment of probabilities to logical statements (such as the probability that the trillionth digit of $\pi$ is $7$). The authors construct a \emph{logical inductor,} an algorithm that keeps track of probabilities for every logical statement. Each day, the logical inductor is presented with a proof of a logical fact, and the inductor updates its probabilities in light of the proof. The inductor's strategy resembles online learning: it assigns weights to each algorithm that assigns probabilities to logical statements, and updates the weights in light of the proofs that it receives. While this work bears resemblance to ours, we are ultimately concerned with \emph{deductive} reasoning about the values of mathematical quantities (or the truth values of logical statements) based on heuristic arguments. By contrast, a logical inductor uses \emph{inductive} reasoning, because it assigns probabilities to logical statements by placing higher weights on algorithms that have done well so far.

Finally, \citet{gowers23coincidence} posits a \emph{no-coincidence principle:} ``If an apparently outrageous coincidence happens in mathematics, then there is a reason for it.'' This intuition is an important motivation for our overall pursuit of heuristic estimation. In our terms, the no-coincidence principle states that every mathematical expression can be adequately ``explained'' to $\GG$ with some heuristic argument. See \citet{christiano2022formalizing} for further discussion.

\section{The subjective properties: Iterated estimation and error orthogonality} \label{sec:iterated_estimation}
\subsection{Motivation and definitions}
We begin by defining and motivating the iterated estimation property.

\begin{defin}
    A heuristic estimator $\GG$ satisfies the \emph{iterated estimation} property if for all expressions $Y$ and for all sets of arguments $\Pi$ and $\Pi' \subseteq \Pi$, we have
    \begin{equation} \tag{\ref{eq:iterated_est}}
        \GG(\GG(Y \mid \Pi) \mid \Pi') = \GG(Y \mid \Pi').
    \end{equation}
\end{defin}

\begin{example}
    Let $Y, \pi_{101}, \pi_{102}$ be as in our running example from Section~\ref{sec:example}. Let $\Pi = \{\pi_{101}, \pi_{102}\}$ and $\Pi' = \{\pi_{101}\}$. We have $\GG(Y \mid \pi_{101}) = 90.5$ (see equation~\ref{eq:pi_101}). The iterated estimation property states that $\GG(\GG(Y \mid \pi_{101}, \pi_{102}) \mid \pi_{101})$ is also equal to $90.5$. In other words, after $\GG$ has learned $\pi_{101}$ (but not yet $\pi_{102}$), its estimate of what its belief of $Y$ \emph{will be} after learning $\pi_{102}$ is its current estimate of $Y$, namely $90.5$.\footnote{Why shouldn't $\GG$ process and incorporate $\pi_{102}$ when estimating what it will believe once it sees $\pi_{102}$? As we will see, processing and merging arguments is in general very nontrivial. We only want $\GG$ to incorporate the arguments that it is given, and not attempt to guess at what its estimate would be if it were to incorporate additional arguments.}
\end{example}

The name ``iterated estimation'' comes from the \emph{law of iterated expectations.} In its simplest form, this law states that for random variables $X, Y$, we have $\EE{\EE{X \mid Y}} = \EE{X}$. More generally and more formally, the law states that on a probability space $(\Omega, \mathcal{F}, \PP)$ with $\sigma$-algebras $\mathcal{H}' \subseteq \mathcal{H} \subseteq \mathcal{F}$, for any integrable random variable $Y$ we have
\begin{equation} \label{eq:iterated_exp}
    \EE{\EE{Y \mid \mathcal{H}} \mid \mathcal{H}'} = \EE{Y \mid \mathcal{H}'},
\end{equation}
where $=$ denotes equality of random variables on all of $\Omega$.

Intuitively, $\mathcal{H}$ is finer than $\mathcal{H}'$, so $\EE{Y \mid \mathcal{H}'}$ is the expectation of $Y$ given some partial information, and $\EE{Y \mid \mathcal{H}}$ is the expectation of $Y$ given more information. Thus, the law of iterated expectations states if someone has partial information about $Y$, then their estimate of what their estimate of $Y$ would be if they had additional information should be their current estimate of $Y$. Put otherwise, they should not be able to predict how they will update their estimate upon learning additional information.

The intuition behind equation~\ref{eq:iterated_est} is essentially the same. $\GG(Y \mid \Pi')$ is $\GG$'s estimate of $Y$ given some set of arguments, and $\GG(Y \mid \Pi)$ is $\GG$'s estimate of $Y$ given those arguments plus additional ones. Thus, in the expression $\GG(\GG(Y \mid \Pi) \mid \Pi')$, the outer $\GG$ estimates what its estimate of $Y$ would be if it considered the additional arguments in $\Pi$, but the outer $\GG$ does not itself consider those arguments. Its estimate of $\GG(Y \mid \Pi)$ should thus be equal to its estimate of $Y$ given the arguments it does consider, i.e.\ $\GG(Y \mid \Pi_1)$. Put otherwise, $\GG$ should not be able to predict how it will update upon considering the arguments in $\Pi \setminus \Pi_1$.\\

There is a stronger form of the law of iterated expectations, which does not have a canonical name but which we will call the projection law:\footnote{Random variables with finite second moments form a Hilbert space with inner product $\angles{X, Y} := \EE{XY}$. The expectation of a random variable $Y$ conditioned on a sigma-algebra $\mathcal{H}$ corresponds to the orthogonal projection of $Y$ onto the subspace of $\mathcal{H}$-measurable random variables (see e.g.\ \citet{sidak57relations}). The projection law can be derived from this fact, hence our choice of name.}
\begin{prop}[{Projection law of conditional expectations, see e.g.\ \citet[Chapter 3]{moshayedi2022conditional}}] \label{prop:projection_law}
    Let $(\Omega, \mathcal{F}, \PP)$ be a probability space with $\sigma$-sub-algebras $\mathcal{H}' \subseteq \mathcal{H} \subseteq \mathcal{F}$. Let $X, Y$ be random variables satisfying $\EE{X^2}, \EE{Y^2} < \infty$. Then
    \begin{equation} \label{eq:projection_law}
        \EE{(Y - \EE{Y \mid \mathcal{H}}) \cdot \EE{X \mid \mathcal{H}} \mid \mathcal{H}'} = 0.
    \end{equation}
\end{prop}

Note that equation~\ref{eq:projection_law} simplifies to equation~\ref{eq:iterated_exp} in the case where $X = 1$. Moreover, the intuition for equation~\ref{eq:projection_law} is similar to the intuition for equation~\ref{eq:iterated_exp}. Equation~\ref{eq:projection_law} states that if all you know is the partial information given by $\mathcal{H}'$, then $Y - \EE{Y \mid \mathcal{H}}$ and $\EE{X \mid \mathcal{H}}$ are uncorrelated over your uncertainty about the state of the world $\omega \in \Omega$. This is true because $Y - \EE{Y \mid \mathcal{H}}$ is your error in estimating $Y$ after learning the information given by $\mathcal{H}$, while $\EE{X \mid \mathcal{H}}$ is a random variable that only depends on the information given by $\mathcal{H}$. Thus, learning the value of $\EE{X \mid \mathcal{H}}$ cannot cause you to update your estimate of $Y - \EE{Y \mid \mathcal{H}}$ away from $0$. Again, the core intuition is that it should be impossible to predict your error (i.e.\ $Y - \EE{Y \mid \mathcal{H}}$) based on information that you already know (i.e.\ $\EE{X \mid \mathcal{H}}$).\\

Just as the iterated estimation property is inspired by the law of iterated expectations, the error orthogonality property is inspired by the projection law.
\begin{defin} \label{def:error_orthogonality}
    A heuristic estimator $\GG$ satisfies the \emph{error orthogonality} property if for all expressions $X, Y$ and for all sets of arguments $\Pi$ and $\Pi_1, \Pi_2 \subseteq \Pi$, we have
    \begin{equation} \tag{\ref{eq:error_orth}}
        \GG((Y - \GG(Y \mid \Pi)) \cdot \GG(X \mid \Pi_1) \mid \Pi_2) = 0.
    \end{equation}
\end{defin}

\begin{example}
    Let $Y, \pi_{101}$ be as in our running example from Section~\ref{sec:example}. Let $X$ be the expression $d_6(\sqrt{101})$ (whose value is $5$). We consider a few situations:
    \begin{enumerate}[label=(\arabic*)]
        \item (Example) If $\Pi = \Pi_1 = \Pi_2 = \{\pi_{101}\}$, then the outer $\GG$ (which knows that $X = 5$) would perhaps reason along the following lines: it knows that $\GG(X \mid \Pi_1) = 5$. This is a constant, so is subjectively uncorrelated with $Y - \GG(Y \mid \Pi)$. The outer $\GG$ also has the same estimate for $Y$ as for $\GG(Y \mid \Pi)$. Thus, its estimate of the entire expression is $0$.
        \item (Example) If $\Pi = \Pi_1 = \{\pi_{101}\}$ and $\Pi_2 = \emptyset$, then the outer $\GG$ does not know the exact value of $\GG(X \mid \Pi_1)$. However, it believes $\GG(X \mid \Pi_1)$ and $Y - \GG(Y \mid \Pi)$ to be subjectively uncorrelated, because it knows that $\Pi$ includes $\Pi_1$. Thus, given the outer $\GG$'s state of knowledge, the best estimate of $Y - \GG(Y \mid \Pi)$ is zero for all possible values of $\GG(X \mid \Pi_1)$. Thus, its estimate of the entire expression is $0$.
        \item (Non-example) If $\Pi = \Pi_1 = \emptyset$ and $\Pi_2 = \{\pi_{101}\}$ (thus violating the premise that $\Pi_2 \subseteq \Pi$), then the outer $\GG$'s estimate of $Y$ is $90.5$, whereas it knows that $\GG(Y \mid \Pi) = 90$ and $\GG(X \mid \Pi_1) = 4.5$. Thus, its estimate of the entire expression is $2.25$.
        \item (Non-example) If $\Pi = \Pi_2 = \emptyset$ and $\Pi_1 = \{\pi_{101}\}$ (thus violating the premise that $\Pi_1 \subseteq \Pi$), then the outer $\GG$ still does not know the exact value of $\GG(X \mid \Pi_1)$. In this situation, though, it believes $\GG(X \mid \Pi_1)$ and $Y - \GG(Y \mid \Pi)$ to be subjectively \emph{correlated.} This is because the outer $\GG$ knows for sure that $\GG(Y \mid \Pi) = 90$; meanwhile, both $Y$ and $\GG(X \mid \Pi_1)$ depend on the value of $X$, and thus $Y$ and $\GG(X \mid \Pi_1)$ are positively correlated over the outer $\GG$'s uncertainty about $X$. Thus, its estimate of the entire expression is positive.
    \end{enumerate}
\end{example}

Why would we like $\mathbb{G}$ to satisfy the error orthogonality property? Recall our comparison of heuristic estimates to subjective expected values: specifically, let us consider the output of $\GG$ as a subjective expectation given a state of knowledge. If the quantity in equation~\ref{eq:error_orth} is not zero, this means that an observer with state-of-knowledge $\Pi_2$ believes $Y - \GG(Y \mid \Pi)$ and $\GG(X \mid \Pi_1)$ to be subjectively correlated over the observer's uncertainty.\footnote{For two quantities $A$ and $B$, the quantity $\GG(AB \mid \pi) - \GG(A \mid \pi)\GG(B \mid \pi)$ represents the \emph{subjective covariance} between $A$ and $B$ according to an observer whose state of knowledge is described by $\pi$. In this case, we take $A = Y - \GG(Y \mid \Pi)$, $B = \GG(X \mid \Pi_1)$, and $\pi = \Pi_2$, and note that -- at least if $\GG$ satisfies linearity and iterated estimation -- we have $\GG(A \mid \pi) = 0$.} In other words: the observer believes that the subjective estimate of $X$ given state-of-knowledge $\Pi_1$ is predictive of the error in the subjective estimate of $Y$ given state-of-knowledge $\Pi$. However, any such prediction should have already been factored into the estimate $\GG(Y \mid \Pi)$.

In other words, just like the iterated estimation property, the error orthogonality property is about $\GG$'s errors being unpredictable to $\GG$. However, the error orthogonality property states that $\GG$'s error in estimating $Y$ given $\Pi$ should be subjectively unpredictable even given $\GG$'s estimate of some quantity $X$ based on $\Pi_1 \subseteq \Pi$. This parallels the sense in which the projection law generalizes the law of iterated expectations.

\subsection{Challenges with the subjective approach} \label{sec:challenges}
Although there are reasons to find the iterated estimation and error orthogonality properties compelling, there are challenges with using these properties as stated to seek a reasonable heuristic estimator. These challenges come primarily from the fact that the properties concern $\GG$'s estimates of its own output.

The first challenge: it seems plausible that the iterated estimation and error orthogonality properties could be satisfied ``by fiat'': for example, $\GG$ could check whether it is estimating the quantity $\GG(Y \mid \Pi)$ given $\Pi' \subseteq \Pi$, and then -- if so -- simply compute $\GG(Y \mid \Pi')$ and output the result. Although this behavior would satisfy the iterated estimation property, we do not want $\GG$ to special-case expressions of this form. Instead, we want $\GG$ to satisfy the iterated estimation property as a consequence of its more general behavior.\footnote{As an analogy, consider a proof verifier $V$ that takes as input a mathematical statement $x$ and a purported proof $\pi$, and outputs $1$ (accept) or $0$ (reject) depending on whether $\pi$ proves $x$. Let $s(x, \pi)$ be the statement ``If $V(x, \pi) = 1$, then $x$.'' For every $(x, \pi)$, there is a proof $\pi'$ of $s(x, \pi)$ (specifically: if $V(x, \pi) = 1$, then $\pi$ proves $x$ and thus $s(x, \pi)$; and if $V(x, \pi) = 0$, then the computational trace of $V$ on $(x, \pi)$ shows that $V(x, \pi) = 0$ and thus proves $s(x, \pi)$). However, $V$ should not treat the input $(s(x, \pi), \pi')$ as a special case; instead, $V$ should verify that $\pi'$ proves $s(x, \pi)$ just as it would verify any other proof.}

It is unclear whether this special-casing approach would work if $\GG$ also needed to satisfy other properties, such as respect for proofs (see Section~\ref{sec:new_properties}). The question of whether the iterated estimation and error orthogonality properties meaningfully constrain the solution space of heuristic estimators that also satisfy linearity and respect for proofs -- or whether a heuristic estimator that satisfies linearity and respect for proofs can be trivially modified to satisfy these two new properties via special-casing -- is an intriguing direction for future work.

The second challenge: the fact that these two properties concern $\GG$'s estimates of its own output makes it difficult to use these properties to reason about $\GG$. If our goal is to find a reasonable heuristic estimator $\GG$, it is most useful to have properties that pin down $\GG$'s outputs on simple inputs. The iterated estimation and error orthogonality properties are not helpful in this regard, because the simplest possible equality that is derivable from either property still involves a mathematical expression that includes the code of $\GG$ as part of the expression. Furthermore, without knowing $\GG$'s code, a constraint that involves $\GG$'s behavior on its own code is less useful.

For these two reasons, we are interested in more grounded variants of the iterated estimation and error orthogonality properties: ones that still capture the key intuition that $\GG$'s errors ought not be predictable, but that do not involve nested $\GG$'s. In the next section, we turn to one such approach.

\section{Accuracy as an objective measure of error unpredictability} \label{sec:multiacc}
\subsection{Motivation and definitions} \label{sec:multiacc_intro}
As we have discussed, the iterated estimation and error orthogonality properties state that $\GG$'s errors should be \emph{subjectively} unpredictable, i.e.\ unpredictable to $\GG$ itself. In this section, we discuss as an alternative property that we call \emph{accuracy} with respect to a predictor $X$, or \emph{multiaccuracy} with respect to a set of predictors $S$. (Multi)accuracy states that $\GG$'s errors should be \emph{objectively} unpredictable: that is, zero-mean over a specified distribution of expressions that are similar to the one being estimated.

The term ``multiaccuracy'' originates in the algorithmic fairness literature, where it is used to describe a predictor that appears unbiased to a given set of statistical tests \citep{hebert18multicalibration, kim19multiaccuracy}. We adapt this definition for our purposes.

\begin{defin} \label{def:multiacc}
    Let $\mathcal{Y}$ be a space of real-valued mathematical expressions, and let $\mathcal{D}$ be a probability distribution over $\mathcal{Y}$ such that $\EE[Y \sim \mathcal{D}]{Y^2} < \infty$.\footnote{When taking expected values as in Equation~\ref{eq:multiacc}, we will elide the distinction between the expression $Y$ and its value.} Let $X: \mathcal{Y} \to \RR$ be a random variable such that $\EE[Y \sim \mathcal{D}]{X^2} < \infty$. An estimator\footnote{For now, we will assume that $f$ is a deterministic function of $Y$. In Section~\ref{sec:approx}, we will discuss sampling-based estimators.} $f: \mathcal{Y} \to \RR$ is \emph{$X$-accurate} over $\mathcal{D}$ if
    \begin{equation} \label{eq:multiacc}
        \EE[Y \sim \mathcal{D}]{(Y - f(Y)) X} = 0.
    \end{equation}
    We say that $f$ is \emph{self-accurate} over $\mathcal{D}$ if $f$ is $f$-accurate over $\mathcal{D}$. For a set $S$ of random variables, we say that $f$ is \emph{$S$-multiaccurate} over $\mathcal{D}$ if $f$ is $X$-accurate over $\mathcal{D}$ for all $X \in S$.
\end{defin}

\begin{example} \label{ex:acc}
    Let $\mathcal{Y}$ be the space of expressions of the form $2 \cdot c_1 + 3 \cdot c_2$, where $c_1, c_2 \in \RR$. (For example, the expression $2 \cdot 0.7 + 3 \cdot -5$ belongs to $\mathcal{Y}$.) Let $\mathcal{D}$ be the distribution over $\mathcal{Y}$ obtained by selecting $c_1, c_2$ independently from $\mathcal{N}(0, 1)$. In Figure~\ref{fig:accuracy_venn}, we classify several estimators of $Y \in \mathcal{Y}$ based on whether they are $1$-accurate, $c_1$-accurate, and self-accurate over $\mathcal{D}$.

    \begin{figure}[ht]
        \centering
        \includegraphics[scale=1]{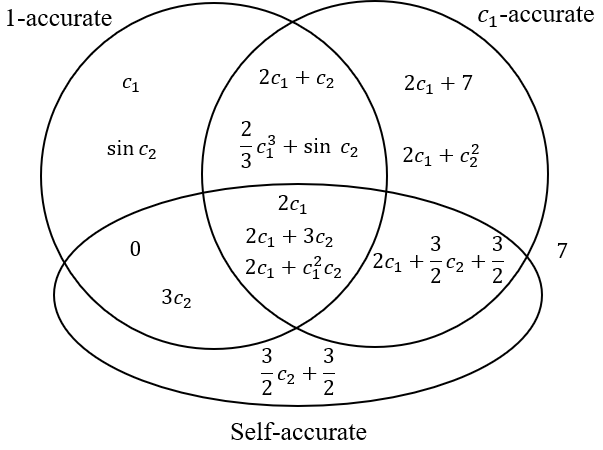}
        \caption{Let $\mathcal{Y}$ be the space of expressions of the form $2 \cdot c_1 + 3 \cdot c_2$, where $c_1, c_2 \in \RR$, and let $\mathcal{D}$ be the distribution over $\mathcal{Y}$ obtained by selecting $c_1, c_2$ independently from $\mathcal{N}(0, 1)$. This figure classifies estimators of $Y \in \mathcal{Y}$ based on whether they are $1$-accurate, $c_1$-accurate, and self-accurate over $\mathcal{D}$.}
        \label{fig:accuracy_venn}
    \end{figure}

    For example, the estimator $f(Y) = 3c_2$ is $1$-accurate over $\mathcal{D}$, meaning that it has the correct mean:
    \[\EE[Y \sim \mathcal{D}]{(Y - f(Y)) \cdot 1} = \EE[c_1, c_2 \sim \mathcal{N}(0, 1)]{2c_1} = 0.\]
    It is also self-accurate:
    \[\EE[Y \sim \mathcal{D}]{(Y - f(Y)) \cdot f(Y)} = \EE[c_1, c_2 \sim \mathcal{N}(0, 1)]{2c_1 \cdot 3c_2} = 0.\]
    However, it is not $c_1$-accurate:
    \[\EE[Y \sim \mathcal{D}]{(Y - f(Y)) \cdot c_1} = \EE[c_1, c_2 \sim \mathcal{N}(0, 1)]{2c_1 \cdot c_1} = 2 \neq 0.\]
\end{example}

Although the expressions in the space $\mathcal{Y}$ of Example~\ref{ex:acc} are easy to compute, for the remainder of this work, $\mathcal{Y}$ will typically contain expressions whose values cannot be efficiently estimated. (For example, $Y$ might be the permanent of a matrix -- see Section~\ref{sec:perm}.) Meanwhile, $X$ is best thought of an efficiently-computable predictor\footnote{In the sense of predictor variables in statistical models.} of $Y$, and $f(Y)$ as an efficiently-computable estimator of $Y$. While the type signatures of $X$ and $f$ are the same -- they are random variables on $\mathcal{Y}$, or, put otherwise, functions from $\mathcal{Y}$ to $\RR$ -- they sometimes play different roles. In particular, the predictor $X$ (or set of predictors $S$) is typically given, and we are interested in finding an estimator $f(Y)$ that is $X$-accurate (or $S$-multiaccurate).

If $X = 1$, then equation~\ref{eq:multiacc} reduces to the condition that $f$ has the correct mean: $\EE[\mathcal{D}]{f(Y)} = \EE[\mathcal{D}]{Y}$. If $X = f(Y)$, then equation~\ref{eq:multiacc} reduces to $\EE[\mathcal{D}]{(Y - f(Y))f(Y)} = 0$, i.e.\ that the error of $f$ is uncorrelated with the value of $f$. As Proposition~\ref{prop:multiacc} will show, an estimator that does not satisfy this condition can be improved (from the standpoint of expected squared error) simply by being multiplied by an appropriate constant, whereas an estimator that satisfies this condition cannot be improved in this way.

\begin{remark} \label{remark:nontransitive}
    Accuracy is not transitive: if $f$ is $g$-accurate and $g$ is $h$-accurate, it does not follow that $f$ is $h$-accurate. One example: $Y \sim \mathcal{N}(0, 1)$, $f(Y) = Y + 1$, $g(Y) = Y$, and $h(Y) = 1$.
\end{remark}

\begin{prop} \label{prop:multiacc}
    Let $\mathcal{D}$, $S$, $f$ be as in Definition~\ref{def:multiacc}. The following are equivalent:
    \begin{enumerate}[label=(\arabic*)]
        \item \label{item:multiacc} $f$ is $S$-multiaccurate over $\mathcal{D}$.
        \item \label{item:onex} For all $X \in S$,  among all estimators that differ from $f$ by a constant multiple of $X$, $f$ itself is optimal. That is,
        \begin{equation} \label{eq:onex}
            \arg \min_c \EE[Y \sim \mathcal{D}]{(Y - (f(Y) + cX))^2} = 0.
        \end{equation}
        \item \label{item:multix} For all $n$, for all $X_1, \dots, X_m \in S$, among all estimators that differ from $f$ by a linear combination of the $X_i$'s, $f$ itself is optimal. That is,
        \begin{equation} \label{eq:multix}
            \arg \min_{(c_1, \dots, c_m)} \EE[Y \sim \mathcal{D}]{\parens{Y - \parens{f(Y) + \sum_{i = 1}^m c_i X_i}}^2} = \vect{0}.
        \end{equation}
    \end{enumerate}
\end{prop}

\begin{proof}
    Clearly, \ref{item:multix} implies \ref{item:onex}.
    
    To see that \ref{item:onex} implies \ref{item:multiacc}, assume \ref{item:onex}. Choose $X \in S$. Either $X$ is zero almost surely (in which case equation~\ref{eq:multiacc} is satisfied for $X$), or the expression in equation~\ref{eq:onex} is a quadratic function of $c$ with a minimum at $c = \frac{\EE[\mathcal{D}]{(Y - f(Y))X}}{\EE[\mathcal{D}]{X^2}}$. By assumption, this quantity is zero, which means that $\EE[\mathcal{D}]{(Y - f(Y))X} = 0$.
    
    To see that \ref{item:multiacc} implies \ref{item:multix}, assume \ref{item:multiacc}. Choose $X_1, \dots, X_m \in S$. The expression in equation~\ref{eq:multix} is a quadratic function of $c_1, \dots, c_m$ that is never negative, which means that it attains a global minimum at any point $\vect{c} = (c_1, \dots, c_m)$ where it has gradient zero. Thus, it suffices to show that the expression has gradient $\vect{0}$ at $\vect{c} = \vect{0}$. For any $i$, the derivative of the expression with respect to $c_i$ is
    \[\EE[\mathcal{D}]{2 \parens{\sum_i c_iX_i + f(Y) - Y} X_i}.\]
    At $\vect{c} = \vect{0}$, this simplifies to $2 \EE[\mathcal{D}]{(f(Y) - Y)X_i}$, which is indeed zero by assumption.
\end{proof}

Proposition~\ref{prop:multiacc} suggests important similarities between multiaccuracy and linear regression. In statistics, the \emph{ordinary least squares} (OLS) estimator of a random variable $Y$ in terms of predictors $X_1, \dots, X_m$ is the linear combination $\sum_i \beta_i X_i$ that minimizes the expected squared distance from $Y$.\footnote{Often, the OLS estimator is defined as the best \emph{affine} combination of $X_1, \dots, X_m$, i.e.\ the best estimator of the form $\alpha + \sum_i \beta_i X_i$. In our discussions, we will not allow arbitrary constants. However, often one of the $X_i$'s will itself be a constant (e.g.\ in discussions of $1$-accuracy). When one of the predictors is a constant, the two OLS notions are the same.} See e.g.\ \citet{hastie09elements} for an introduction to OLS regression.

In particular, the OLS estimator of $Y$ in terms of $X_1, \dots, X_m$ satisfies property~\ref{item:multix} above (by definition), and is thus $S$-multiaccurate (where $S = \{X_1, \dots, X_m\}$). While there can be many $S$-multiaccurate estimators of $Y$ over $\mathcal{D}$ (see Example~\ref{ex:acc}), the OLS estimator is the only $S$-multiaccurate estimator that is a linear combination of $X_1, \dots, X_m$.\footnote{If $X_1, \dots, X_m$ are linearly dependent, then there may be multiple ways of expressing this estimator as a linear combination of $X_1, \dots, X_m$.} Furthermore, the OLS estimator is always self-accurate: since the OLS estimator cannot be improved by adding any linear combination of the $X_i$'s, then certainly it cannot be improved by adding a multiple of itself (as the OLS estimator is itself a linear combination of the $X_i$'s).\\

The notions introduced above give a potential path forward for defining properties of heuristic estimators similar to iterated estimation and error orthogonality, but without the need for nested $\GG$'s. The basic idea is to substitute $\GG(Y \mid \Pi)$ for $f(Y)$ in equation~\ref{eq:multiacc}:
\[\EE[Y \sim \mathcal{D}]{(Y - \GG(Y \mid \Pi))X} = 0.\]
In order to make sense of this expression, we need to somewhat reconceptualize the notion of a heuristic argument. In Section~\ref{sec:iterated_estimation}, we regarded the heuristic argument $\pi$ in the expression $\GG(Y \mid \pi)$ as a particular computation that helps to estimate $Y$ (perhaps $\pi$ partially computes $Y$). Now that we are taking an expectation over different expressions $Y$, we must allow $\pi$ to depend on $Y$. Thus, in this section, we will be thinking of $\pi$ as a computation that depends on $Y$ (i.e.\ a computer program that takes $Y \in \mathcal{Y}$ as input). When $\GG$ is given $Y$ and $\pi$ as input, it runs $\pi$ on $Y$ as part of estimating $Y$. Similarly, if $\Pi = \{\pi_1, \dots, \pi_m\}$, then $\GG(Y \mid \Pi)$ runs $\pi_1, \dots, \pi_m$ on $Y$. (See details in footnote.\footnote{We can unify this view of $\pi$ (as a program that $\GG$ runs on $Y$) with our earlier view of $\pi$ (as a partial computation) as follows: given a program $P$ that takes elements of $\mathcal{Y}$ as input, define $\pi_P(Y)$ to be the execution trace of $P$ run on $Y$. (The execution trace should include $P$'s code.) Given a set of programs $\mathcal{P} = \{P_1, \dots, P_m\}$, define $\Pi_{\mathcal{P}}(Y) := \{\pi_{P_1}(Y), \dots, \pi_{P_m}(Y)\}$. For a given $\Pi = \Pi_{\mathcal{P}}$, when we write $\EE[Y \sim \mathcal{D}]{(Y - \GG(Y \mid \Pi))X} = 0$, we mean that $\EE[Y \sim \mathcal{D}]{(Y - \GG(Y \mid \Pi_{\mathcal{P}}(Y))) X} = 0$. \label{footnote:unify_pi}})

We define (multi)accuracy for heuristic estimators as follows.

\begin{defin} \label{def:g_multiacc}
    Let $\mathcal{Y}$, $\mathcal{D}$ be as in Definition~\ref{def:multiacc}. Let $\GG$ be a heuristic estimator and $X: \mathcal{Y} \to \RR$ be a random variable. A set of heuristic arguments $\Pi_X$ \emph{makes $\GG$ be $X$-accurate} over $\mathcal{D}$ if for all $\Pi \supseteq \Pi_X$, $\GG(Y \mid \Pi)$ is an $X$-accurate estimator over $\mathcal{D}$ (in the sense of Definition~\ref{def:multiacc}) -- that is,
    \begin{equation} \label{eq:g_acc}
        \EE[Y \sim \mathcal{D}]{(Y - \GG(Y \mid \Pi))X} = 0.
    \end{equation}
    We say that $\GG$ is \emph{$X$-accurate} over $\mathcal{D}$ if there exists a short\footnote{More precisely, the computations that $\Pi_X$ requires $\GG$ to run must be polynomial-time in the time that it takes to compute $X$. Or in other words, if $\Pi_X$ corresponds to a set of programs $\mathcal{P}_X = \{P_1, \dots, P_m\}$ (as in footnote~\ref{footnote:unify_pi}), then each $\pi_{P_i}(Y)$ must have length that is polynomial in the time that it takes to compute $X$. We require this so that $\Pi_X$ is not allowed to force $\GG$ to compute $Y$ exactly.} $\Pi_X$ that makes $\GG$ be $X$-accurate over $\mathcal{D}$. For a set $S$ of random variables, we say that $\GG$ is \emph{$S$-multiaccurate} over $\mathcal{D}$ is $\GG$ is $X$-accurate over $\mathcal{D}$ for all $X \in S$.
\end{defin}

In the special case of $X = 1$, equation~\ref{eq:g_acc} is quite similar to equation~\ref{eq:iterated_est} (iterated estimation), the key difference being that the outer $\GG$ is replaced by an expectation over $\mathcal{D}$.

We also define self-accuracy for $\GG$:
\begin{defin}
    Let $\mathcal{Y}$, $\mathcal{D}$ be as in Definition~\ref{def:multiacc}. A heuristic estimator $\GG$ is \emph{self-accurate} over $\mathcal{D}$ if for every set of heuristic arguments $\Pi$, $\Pi$ makes $\GG$ be $\GG(Y \mid \Pi)$-accurate -- that is, if for all $\Pi' \supseteq \Pi$, we have
    \begin{equation} \label{eq:g_self_acc}
        \EE[Y \sim \mathcal{D}]{(Y - \GG(Y \mid \Pi')) \cdot \GG(Y \mid \Pi)} = 0.
    \end{equation}
\end{defin}

Equation~\ref{eq:g_self_acc} is quite similar to equation~\ref{eq:error_orth} (error orthogonality), again with the key difference being that the outer $\GG$ is replaced by an expectation over $\mathcal{D}$.

\subsection{Multiaccuracy as a constraint on argument merges} \label{sec:g_multiacc_ols}
It can be useful to think of multiaccuracy as a constraint on merging estimates based on different arguments. Concretely, given a set of predictors $S = \{X_1, \dots, X_m\}$, suppose that $\GG$ is $S$-multiaccurate over $\mathcal{D}$; let $\Pi_{X_i}$ be an argument that makes $\GG$ be $X_i$-accurate over $\mathcal{D}$. Then $\GG(Y \mid \Pi_{X_1}, \dots, \Pi_{X_m})$ is an $\{X_1, \dots, X_m\}$-multiaccurate estimator over $\mathcal{D}$: for each $i$, $\GG$'s error is orthogonal to $X_i$.

Is it possible for $\GG$ to be $S$-multiaccurate for a rich set of predictors $S$, while still running in polynomial time in the size of the arguments provided? One reason for hope is that estimates may be merged using linear regression:

\begin{prop} \label{prop:multiacc_linreg}
    Let $\mathcal{Y}$, $\mathcal{D}$ be as in Definition~\ref{def:g_multiacc}. Let $\GG$ be a heuristic estimator, and suppose that $\GG$ merges arguments using OLS regression: that is, for all sets of arguments $\Pi = \{\pi_1, \dots, \pi_m\}$, we have
    \begin{equation} \label{eq:multiacc_linreg}
        \GG(Y \mid \Pi) = \begin{pmatrix}\EE[\mathcal{D}]{Y \GG(Y \mid \pi_1)} \\ \vdots \\ \EE[\mathcal{D}]{Y \GG(Y \mid \pi_m)}\end{pmatrix}^\top M^+ \begin{pmatrix}\GG(Y \mid \pi_1) \\ \vdots \\ \GG(Y \mid \pi_m)\end{pmatrix},
    \end{equation}
    where $M^+$ is the pseudoinverse of the Gram matrix $M := \parens{\EE[\mathcal{D}]{\GG(Y \mid \pi_i) \GG(Y \mid \pi_j)}}_{i, j = 1}^m$. Then $\GG$ is self-accurate over $\mathcal{D}$.
\end{prop}

(Why is it important that $\GG$ is self-accurate? If $\GG$ is self-accurate, then $\GG$ is $\GG(Y \mid \pi)$-accurate for every $\pi$. Thus, if the set of predictors $\GG(Y \mid \pi)$ is rich, then $\GG$ will be multiaccurate over a rich set of predictors.)

\begin{proof}
    We must show that for all $\Pi$ and for all $\Pi' \supseteq \Pi$, $\GG(Y \mid \Pi')$ is a $\GG(Y \mid \Pi)$-accurate estimator of $Y$ over $\mathcal{D}$. To see this, note that since $\GG(Y \mid \Pi')$ is the OLS estimator of $Y$ with predictors $\{\GG(Y \mid \pi'): \pi' \in \Pi'\}$, no linear combination of $\{\GG(Y \mid \pi): \pi \in \Pi\}$ can be added to $\GG(Y \mid \Pi')$ to produce an improved estimator. By Proposition~\ref{prop:multiacc}, it follows that $\GG(Y \mid \Pi')$ is a $\GG(Y \mid \Pi)$-accurate estimator.
\end{proof}

Unfortunately, Proposition~\ref{prop:multiacc_linreg} does not provide a straightforward, computationally efficient way to merge heuristic estimates. This is because computing the Gram matrix $M$ can be $\mathsf{\#P}$-hard even in simple cases, as we will see in Section~\ref{sec:hardness}. Further, even \emph{approximating} the covariances by sampling from $\mathcal{D}$ can require exponentially many samples.\footnote{As a simple example, suppose that $Y \sim \mathcal{D}$ is parameterized by i.i.d.\ standard Gaussians, and that $\GG(Y \mid \pi_1) = \GG(Y \mid \pi_2)$ is a product of $n$ of those variables (see Section~\ref{sec:hafnian} for an example of such a $\mathcal{D}$). Then $\EE[\mathcal{D}]{\GG(Y \mid \pi_1) \GG(Y \mid \pi_2)} = 1$, but the variance of $\GG(Y \mid \pi_1) \GG(Y \mid \pi_2)$ is $3^n - 1$. Thus, in order to estimate the expectation purely by sampling, many samples are needed.}

In the next two sections, we will explore accurate estimation in two specific contexts: estimating the expected product of jointly normal random variables, and estimating the permanent of a matrix. Our goal will be to demonstrate some limits to producing estimates that satisfy the accuracy definitions that we presented in this section.

\section{Estimating the product of jointly normal random variables} \label{sec:hafnian}
In this section, we consider the problem of estimating the expected product of $n$ jointly normal random variables. Formally: given an $n \times n$ covariance matrix $\Sigma$, we would like to estimate
\begin{equation} \label{eq:prod_gaussians}
    \EE[(Z_1, \dots, Z_n) \sim \mathcal{N}(\pmb{0}, \Sigma)]{\prod_{i = 1}^n Z_i}.
\end{equation}
This problem is motivated by two considerations. On the one hand, it is one of the simplest estimation problems in which computing (or indeed approximating) the correct answer is computationally intractable. On the other hand, the problem captures the core difficulty of a natural, more general estimation problem: estimating the average output of an arithmetic circuit (a circuit with addition and multiplication gates). Addition gates in an arithmetic circuit are straightforward to handle: after all, the linearity property (see Section~\ref{sec:new_properties}) states that the estimate of a sum should equal the sum of the estimates of each summand. As we are about to see, multiplication is not so straightforward.

It turns out that the expected product in equation~\ref{eq:prod_gaussians} can be written as an exponentially large sum of products of $\frac{n}{2}$ entries of $\Sigma$ (see Theorem~\ref{thm:isserlis}). While this sum cannot be computed efficiently, we consider a class of heuristic estimates that compute the sum over a \emph{subset} of these terms. We then ask whether it is possible to \emph{accurately merge} multiple such heuristic estimates $X_1, \dots, X_m$: that is, to create an estimator that is $\{X_1, \dots, X_m\}$-multiaccurate (over a certain natural distribution)? While we do not fully resolve this question, in Section~\ref{sec:hardness} we show that even for $m = 2$ arguments, finding a multiaccurate \emph{linear} merge (an estimator of the form $\beta_1 X_1 + \beta_2 X_2$) is $\mathsf{\#P}$-hard.

This motivates relaxing our standard of accuracy from \emph{exact} accuracy to \emph{approximate} accuracy. In Section~\ref{sec:approx}, we define approximate accuracy and explore the problem of approximately merging heuristic estimates for the expected product of jointly normal random variables. We demonstrate a randomized algorithm for producing an approximately accurate estimator, although our algorithm only runs in polynomial time if certain assumptions are made about the arguments being merged (roughly speaking, that they do not overlap too strongly). We leave open the question of whether approximately accurate merges are possible in polynomial time in full generality.\\

Our discussion of this estimation problem will require the notion of \emph{pairings} of a set.

\begin{defin}
    A \emph{pairing} of a finite set $K$ is a set of $\abs{K}/2$ unordered pairs $(i, j)$ of elements of $K$, such that each element of $K$ appears in exactly one pair. (We will sometimes call $K$ the \emph{index set}.) We use $\mathcal{P}_2(K)$ to denote the set of all pairings of $K$. (If $\abs{K}$ is odd then $\mathcal{P}_2(K)$ is empty.) To simplify notation, we will write $\mathcal{P}_2(n)$ instead of $\mathcal{P}_2([n])$ for the set of pairings of $\{1, \dots, n\}$.
\end{defin}

We will use the following result, sometimes known as Isserlis' theorem.

\begin{theorem}[\cite{isserlis18product}] \label{thm:isserlis}
    Let $Z_1, \dots, Z_n$ be jointly normal, zero-mean random variables. Then
    \begin{equation} \label{eq:isserlis}
        \EE{Z_1 \dots Z_n} = \sum_{p \in \mathcal{P}_2(n)} \prod_{\{i, j\} \in p} \Cov{Z_i, Z_j}.
    \end{equation}
\end{theorem}

Note that if $n$ is odd then $\mathcal{P}_2(n)$ is empty, so $\EE{Z_1 \dots Z_n} = 0$. For the remainder of this section, we will assume that $n$ is positive and even.

\begin{remark}
    For even $n$, the \emph{hafnian} $\haf(A)$ of a symmetric $n \times n$ matrix $A$ is defined as
    \[\haf(A) := \sum_{p \in \mathcal{P}_2(n)} \prod_{\{i, j\} \in p} A_{i, j}.\]
    Theorem~\ref{thm:isserlis} can be restated as follows: for $Z_1, \dots, Z_n \sim \mathcal{N}(\pmb{0}, \Sigma)$, we have $\EE{Z_1 \dots Z_n} = \haf(\Sigma)$. The hafnian generalizes the permanent (we discuss permanents more in Section~\ref{sec:perm}). In particular, for a square matrix $A$, we have $\perm(A) = \haf \begin{pmatrix}0 & A \\ A^\top & 0\end{pmatrix}$. Computing the permanent of a $0-1$ matrix is $\mathsf{\#P}$-hard \citep{valiant79permanent}, and therefore so is computing the hafnian.
\end{remark}

In this section, we will study accurate estimation of $\haf(\Sigma)$ over the following distribution of covariance matrices $\Sigma$: every off-diagonal entry of $\Sigma$ is chosen from $\mathcal{N}(0, 1)$ independently. (The diagonal entries of $\Sigma$ -- which do not affect the expectation of $Z_1 \dots Z_n$ -- are chosen to be large enough that $\Sigma$ is PSD.) We let $\mathcal{D}_n$ be the induced distribution of the expression $\haf(\Sigma)$. When $n$ is fixed, we suppress the subscript and simply write $\mathcal{D}$.

We now define a class of estimators of the hafnian, with respect to which we will aim to be accurate.
\begin{defin}
    Let $\Sigma$ be an $n \times n$ PSD matrix and let $K \subseteq [n]$. For a pairing $p$ of $K$, we define $X_p := \prod_{\{i, j\} \in p} \Sigma_{i, j}$. For a set of pairings $S \subseteq \mathcal{P}_2(K)$, we define $X_S := \sum_{p \in S} X_p$.
\end{defin}

We being with the following observation.
\begin{claim} \label{claim:xs_accurate}
    For all $S \subseteq \mathcal{P}_2(n)$, $X_S$ is a $1$-accurate and self-accurate estimator of $\haf(\Sigma)$ over $\mathcal{D}$.
\end{claim}

\begin{proof}
    To see that $X_S$ is $1$-accurate, we must show that $\EE[\mathcal{D}]{X_S} = 0$. This is indeed the case, as $X_S$ is a sum of $X_p$ for $p \in S$, and each $X_p$ is zero-mean (because $X_p$ is a product of independent, zero-mean random variables).
    
    To see that $X_S$ is self-accurate, note that $\EE[\mathcal{D}]{X_p X_q} = 0$ for all $p \neq q \in \mathcal{P}_2(n)$. This is because $X_p X_q$ is a product of $n$ not-necessarily-distinct elements of $\Sigma$, at least one of which appears only once in the product. This element is independent of all other elements in the product, and is zero-mean. It follows that $X_S$ is self-accurate, because we have
    \[\EE[\mathcal{D}]{(\haf(\Sigma) - X_S)X_S} = \EE[\mathcal{D}]{\parens{\sum_{p \in \mathcal{P}_2(n)} X_p - \sum_{p \in S} X_p} \parens{\sum_{p \in S} X_p}} = \EE[\mathcal{D}]{\sum_{p \not \in S} X_p \cdot \sum_{p \in S} X_p} = 0.\]
\end{proof}

We also note the following property of estimators $X_S$, which will be useful for proving some of our results.

\begin{claim} \label{claim:set_cov}
    For all sets of pairings $S_1, S_2 \subseteq \mathcal{P}_2(n)$, we have $\EE[\mathcal{D}_n]{X_{S_1} X_{S_2}} = \abs{S_1 \cap S_2}$.
\end{claim}

\begin{proof}
    We have
    \[\EE[\mathcal{D}_n]{X_{S_1} X_{S_2}} = \sum_{p \in S_1} \sum_{q \in S_2} \EE[\mathcal{D}_n]{X_p X_q} = \sum_{p \in S_1 \cap S_2} \EE{X_p^2} = \abs{S_1 \cap S_2}.\]
    The second-to-last step follows from the previously-mentioned fact that $\EE{X_p X_q} = 0$ when $p \neq q$ (see the proof of Claim~\ref{claim:xs_accurate}). The last step follows from the fact that $X_p^2$ is the product of the squares of $n/2$ independent standard Gaussians. Moreover, by setting $S_2 = \mathcal{P}_2(n)$ so that $X_{S_2} = \haf(\Sigma)$, we find that $\EE{X_S \haf(\Sigma)} = \abs{S}$ for all $S$.
\end{proof}

While $X_S$ is not in general efficiently computable for every $S \subseteq \mathcal{P}_2(n)$, it may be efficiently computable if $S$ has a special structure. This can be true even for exponentially-sized sets $S$.

\begin{example} \label{ex:pairing_structure}
    Let $S$ be the set of $3^{n/4}$ pairings that pair $1, 2, 3, 4$ amongst themselves (there are three ways to do so); pair $5, 6, 7, 8$ amongst themselves; and so on. Then
    \[X_S = (\Sigma_{1, 2} \Sigma_{3, 4} + \Sigma_{1, 3} \Sigma_{2, 4} + \Sigma_{1, 4} \Sigma_{2, 3})(\Sigma_{5, 6} \Sigma_{7, 8} + \Sigma_{5, 7} \Sigma_{6, 8} + \Sigma_{5, 8} \Sigma_{6, 7})(\dots).\]
\end{example}

This motivates defining \emph{pairing structures}: a pairing structure is an abstract representation of sets of pairings that is built up recursively. In particular, a pairing structure is defined as either a \emph{union} of pairing structures representing disjoint pairings of the same set, or a \emph{product} of pairing structures representing pairings of disjoint sets. Formally:

\begin{defin} \label{def:pairing_structure}
    A \emph{pairing structure} (denoted as $\mathcal{T}$) over an index set $K$ is a structured representation of a subset of $\mathcal{P}_2(K)$ (we will use $S(\mathcal{T})$ to refer to the set of pairings that $\mathcal{T}$ represents). It is defined as any one of the following:
    \begin{itemize}
        \item If $K = \{i, j\}$ has size $2$, then $\mathcal{T}_K$ is a \emph{base} pairing structure over $K$, and $S(\mathcal{T}) := \mathcal{P}_2(K) = \{\{(i, j)\}\}$.
        \item If $\mathcal{T}_1$ and $\mathcal{T}_2$ are pairing structures over the same index set $K$, such that $S(\mathcal{T}_1)$ and $S(\mathcal{T}_2)$ are \textbf{disjoint sets,} then $\mathcal{T}_1 \cup \mathcal{T}_2$ is a \emph{union} pairing structure over $K$, and $S(\mathcal{T}) := S(\mathcal{T}_1) \cup S(\mathcal{T}_2)$.
        \item If $\mathcal{T}_1$ and $\mathcal{T}_2$ are pairing structures over disjoint index sets $K_1$ and $K_2$, then $\mathcal{T}_1 \otimes \mathcal{T}_2$ is a \emph{product} pairing structure over $K_1 \cup K_2$, and $S(\mathcal{T}) := \{p_1 \cup p_2: p_1 \in S(\mathcal{T}_1), p_2 \in S(\mathcal{T}_2)\}$.
    \end{itemize}

    The \emph{representation size} of a pairing structure $\mathcal{T}$, denoted $\size(\mathcal{T})$, is defined recursively: a base pairing structure has representation size $1$; $\size(\mathcal{T}_1 \otimes \mathcal{T}_2) = \size(\mathcal{T}_1) + \size(\mathcal{T}_2) + 1$; and $\size(\mathcal{T}_1 \cup \mathcal{T}_2) = \size(\mathcal{T}_1) + \size(\mathcal{T}_2) + 1$.

    For convenience, we will define $X_{\mathcal{T}}$ to be shorthand for $X_{S(\mathcal{T})}$.
\end{defin}

For example, we can represent $\mathcal{P}_2(4)$, the set of pairings of $\{1, 2, 3, 4\}$ (of which there are three), as
\[\mathcal{T}_{\{1, 2, 3, 4\}} := (\mathcal{T}_{\{1, 2\}} \otimes \mathcal{T}_{\{3, 4\}}) \cup (\mathcal{T}_{\{1, 3\}} \otimes \mathcal{T}_{\{2, 4\}}) \cup (\mathcal{T}_{\{1, 4\}} \otimes \mathcal{T}_{\{2, 3\}}).\]
That is, $S(\mathcal{T}) = \mathcal{P}_2(4)$. If we define $\mathcal{T}_{\{5, 6, 7, 8\}}$ and so forth similarly, then
\[\mathcal{T} := \mathcal{T}_{\{1, 2, 3, 4\}} \otimes \mathcal{T}_{\{5, 6, 7, 8\}} \otimes \dots \otimes \mathcal{T}_{\{n - 3, n - 2, n - 1, n\}}\]
represents the set of pairings in Example~\ref{ex:pairing_structure}.

For every pairing structure $\mathcal{T}$, $X_{\mathcal{T}}$ can be computed in linear time in $\size(\mathcal{T})$:
\begin{itemize}[noitemsep]
    \item If $\mathcal{T} = \mathcal{T}_{\{i, j\}}$ is a base pairing structure, then $X_{\mathcal{T}} = \Sigma_{i, j}$.
    \item If $\mathcal{T} = \mathcal{T}_1 \cup \mathcal{T}_2$ is a union, then\footnote{This is why we require that $S(\mathcal{T}_1)$ and $S(\mathcal{T}_2)$ be disjoint.} $X_{\mathcal{T}} = X_{\mathcal{T}_1} + X_{\mathcal{T}_2}$.
    \item If $\mathcal{T} = \mathcal{T}_1 \otimes \mathcal{T}_2$ is a product, then $X_{\mathcal{T}} = X_{\mathcal{T}_1} X_{\mathcal{T}_2}$.
\end{itemize}

\subsection{Computing correlations between predictors can be $\mathsf{\#P}$-hard} \label{sec:hardness}
As per Claim~\ref{claim:xs_accurate}, $X_S$ is a $\{1, X_S\}$-multiaccurate estimator of $\haf(\Sigma)$. On the other hand, given two pairing structures $\mathcal{T}$ and $\mathcal{U}$, it is far from obvious whether it is possible to construct a $\{1, X_{\mathcal{T}}, X_{\mathcal{U}}\}$-multiaccurate estimator in time that is polynomial $\size(\mathcal{T}) + \size(\mathcal{U})$. A natural approach would be to find OLS regression coefficients onto predictors $X_{\mathcal{T}}$ and $X_{\mathcal{U}}$ (see Proposition~\ref{prop:multiacc_linreg}). However, this requires computing $\EE[\mathcal{D}]{X_{\mathcal{T}} X_{\mathcal{U}}}$. Theorem~\ref{thm:no_exact_merge} (whose proof we defer to the appendix) shows that doing so is $\mathsf{\#P}$-hard.

\begin{restatable}{theorem}{noexactmerge} \label{thm:no_exact_merge}
    Given an integer $n$ and two pairing structures $\mathcal{T}$ and $\mathcal{U}$ with index set $[n]$ as input, computing $\EE[\mathcal{D}_n]{X_{\mathcal{T}} X_{\mathcal{U}}}$ is $\mathsf{\#P}$-hard.
\end{restatable}

The proof of Theorem~\ref{thm:no_exact_merge} works via a reduction from $\#3SAT$: given a 3CNF $\varphi$, we create pairing trees $\mathcal{T}, \mathcal{U}$ such that the number of satisfying assignments to $\varphi$ is equal to $\abs{S(\mathcal{T}) \cap S(\mathcal{U})}$. Specifically, $\mathcal{T}$ is a product of sub-trees, one for each variable $x_i$ appearing in $\varphi$, which contains one pairing corresponding to setting $x_i = 0$ and another corresponding to setting $x_i = 1$. Meanwhile, $\mathcal{U}$ is a product of sub-trees, one for each clause of $\varphi$, which contains one pairing for each of the seven ways to satisfy the clause.\\

It may be tempting to conclude from Theorem~\ref{thm:no_exact_merge} that there is no efficient way to merge $X_{\mathcal{T}}$ and $X_{\mathcal{U}}$ to produce an estimator that is self-accurate and also $\{1, X_{\mathcal{T}}, X_{\mathcal{U}}\}$-multiaccurate. However, all we have shown is that there is no efficient way to find a \emph{linear} merge of $X_S$ and $X_T$ (i.e.\ an estimator of the form $\beta_S X_S + \beta_T X_T$) with these properties.\footnote{By Proposition~\ref{prop:multiacc}, the only linear merge of $X_\mathcal{T}$ and $X_\mathcal{U}$ that is $\{X_{\mathcal{T}}, X_{\mathcal{U}}\}$-multiaccurate is the OLS regression of the hafnian onto the predictors $X_{\mathcal{T}}$ and $X_{\mathcal{U}}$. In order to compute the OLS regression coefficients, we must be able to compute $\EE[\mathcal{D}_n]{X_{\mathcal{T}} X_{\mathcal{U}}}$. Theorem~\ref{thm:no_exact_merge} showed that doing so is $\mathsf{\#P}$-hard.} We have not ruled out the possibility that a more complicated merge satisfies our desired accuracy conditions. However, we are prepared to conjecture that such a merge does not exist:

\begin{conjecture} \label{conj:no_merge_tu}
    There is no polynomial-time algorithm that, given an integer $n$ and two pairing structures $\mathcal{T}$ and $\mathcal{U}$ with index set $[n]$, outputs a polynomial-time computable estimator $f_{\mathcal{T}, \mathcal{U}}$ that is $\{1, X_{\mathcal{T}}, X_{\mathcal{U}}, f_{\mathcal{T}, \mathcal{U}}\}$-multiaccurate over $\mathcal{D}_n$.
\end{conjecture}

An even stronger conjecture is that no polynomial-time heuristic estimator can accurately merge pairing structure-based arguments.

\begin{conjecture} \label{conj:no_merge_tu_g}
    Given an integer $n$, let $R_n$ be the set of all predictors $X_{\mathcal{T}}$, where $\mathcal{T}$ is a pairing structure over $[n]$. There is no polynomial-time heuristic estimator that is $R_n$-multiaccurate and self-accurate over $\mathcal{D}_n$ for all $n$.
\end{conjecture}

Conjecture~\ref{conj:no_merge_tu_g} is somewhat stronger than Conjecture~\ref{conj:no_merge_tu}: the set of heuristic arguments $\Pi_\mathcal{T}$ corresponding to the pairing structure $\mathcal{T}$ could encode advice for merging $X_{\mathcal{T}}$ with other predictors. However, it seems that a heuristic estimator would need to receive advice about every \emph{pair} of predictors in order to be able to merge them; we do not expect that a heuristic estimator could perfectly merge arguments without such advice. This is why we believe Conjecture~\ref{conj:no_merge_tu_g} to be true.

To recap, we stated a natural estimation problem (finding the expected product of jointly normal random variables, or put otherwise, the hafnian of the covariance matrix) and have defined a natural class of accurate predictors (efficiently computable partial sums of the terms that comprise the hanfian). Then, we showed that the most straightforward way to produce an accurate merge of these predictors is not computationally tractable. Further, we conjecture that \emph{no} computationally tractable, accurate merge exists.

\subsection{Approximating correlations between predictors} \label{sec:approx}
Given the simplicity of our estimation problem and predictors, we consider this computational barrier to be a potential obstacle to using accuracy as a formal desideratum for heuristic estimators. However, perhaps finding an \emph{exactly} accurate merge was too much to hope for. This raises the question: can we merge predictors of the form $X_{\mathcal{T}}$ in a way that is \emph{approximately} accurate with high probability?

In this section, we extend our definition of estimator accuracy to include randomized estimators. Concretely, with $\mathcal{Y}$ and $\mathcal{D}$ as in Definition~\ref{def:multiacc}, define an \emph{estimation algorithm} to be a (possibly randomized) algorithm that outputs an estimator $f$ of $Y$. We define \emph{approximate accuracy} for an estimation algorithm: loosely speaking, an estimation algorithm $A$ is approximately accurate with respect to a predictor $X$ if, with high probability, the estimator $f$ that $A$ outputs is approximately $X$-accurate. Formally:

\begin{defin} \label{def:approx_acc}
    Let $Y, \mathcal{D}$ be as in Definition~\ref{def:multiacc}. Let $f: \mathcal{Y} \to \RR$ be an estimator and $X: \mathcal{Y} \to \RR$ be a predictor. For $\varepsilon \ge 0$, we say that $f$ is \emph{$(\varepsilon, X)$-accurate} over $\mathcal{D}$ if we have
    \begin{equation} \label{eq:approx_acc}
        (\EE[\mathcal{D}]{(Y - f)X})^2 \le \varepsilon^2 \EE[\mathcal{D}]{X^2} \EE[\mathcal{D}]{f^2}.
    \end{equation}
For a set of predictors $S$, we say that $f$ is \emph{$(\varepsilon, S)$-multiaccurate} over $\mathcal{D}$ if $f$ is $(\varepsilon, X)$-accurate over $\mathcal{D}$ for all $X$ in $S$.

Let $A$ be a (possibly randomized) estimation algorithm (which outputs an estimator $f: \mathcal{Y} \to \RR$), and $X: \mathcal{Y} \to \RR$ be a predictor (that may depend on the randomness of $A$\footnote{The purpose of this is to allow e.g.\ $X = f$.}). For $\delta, \varepsilon \ge 0$, we say that $A$ is \emph{$(\delta, \varepsilon, X)$-accurate} over $\mathcal{D}$ if, with probability at least $1 - \delta$ over the randomness of $A$, the output $f$ of $A$ is $(\varepsilon, X)$-accurate over $\mathcal{D}$. For a set of predictors $S$, we say that $f$ is \emph{$(\delta, \varepsilon, S)$-multiaccurate}\footnote{The $\delta$ parameter is somewhat analogous to the $\delta$ in the definition of PAC learning, where a learning algorithm should work with probability $1 - \delta$ over random draws from the distribution. However, in our case, an estimation algorithm could use $Y$ itself as a ``random seed'' (such an algorithm would be be deterministic). Thus, our $\delta$ plays a role of convenience, allowing us to consider randomized estimation algorithms that output very inaccurate estimators with some low probability.} over $\mathcal{D}$ if, with probability at least $1 - \delta$ over the randomness of $A$, the output $f$ of $A$ is $(\varepsilon, X)$-accurate over $\mathcal{D}$ for all $X \in S$.
\end{defin}

We note that our definition is inspired by but not identical to the definition of multiaccuracy in \citet{kim19multiaccuracy}.\footnote{In particular, the right-hand side of equation~\ref{eq:approx_acc} normalizes the required bound in terms of the sizes of $X$ and $f$, while \citet{kim19multiaccuracy} do not normalize. This is because they only consider with predictors $X$ supported on $\{-1, 1\}$ and functions $Y$ supported on $\{0, 1\}$, so normalization is not necessary in their setting.}

\begin{remark} \label{rem:approx_acc_rescale}
    For any constant $c \neq 0$, $f$ is $(\delta, \varepsilon, X)$-accurate if and only if $f$ is $(\delta, \varepsilon, cX)$-accurate. Additionally, $f$ is an $(\delta, \varepsilon, X)$-accurate estimator of $Y$ if and only if $cf$ is an $(\delta, \varepsilon, X)$-accurate estimator of $cY$.
\end{remark}

We offer the following interpretation to motivate Definition~\ref{def:approx_acc}. As discussed previously, (perfect) $X$-accuracy means that no constant multiple of $X$ can be added to $f$ to improve the estimate. In general, the optimal adjustment to $f$ involving adding a constant multiple of $X$ is given by
\[f' := f + \frac{\EE{(Y - f) X}}{\EE{X^2}} X.\]
Our definition of $(\delta, \varepsilon, X)$-accuracy states that, with probability $1 - \delta$, this optimal adjustment of $\frac{\EE{(Y - f) X}}{\EE{X^2}} X$ is small compared to $f$ itself, in terms of second moments.

\begin{remark} \label{rem:self_accuracy}
    When speaking of approximate accuracy, we will typically insist on approximate self-accuracy, i.e.\ on including $f$ itself in the set of predictors with respect to which $f$ should be approximately accurate. Doing so allows us to reject unnecessarily noisy estimators. For example, let $X$ be a predictor and suppose that $f$ is an extremely noisy estimator that has no correlation with $X$ or with $Y$. Then $f$ may be $(\varepsilon, X)$-accurate simply by virtue of the $\EE{f^2}$ term on the right-hand side of equation~\ref{eq:approx_acc} being large.\footnote{Why not instead define $(\varepsilon, X)$-accuracy to mean that $\EE{fX}$ is within a $1 \pm \varepsilon$ factor of $\EE{YX}$, i.e.\ that the amount of $X$ in $f$ is correct to within $\pm \varepsilon$ tolerance? While such a definition would solve this particular problem, the definition is too harsh for predictors $X$ than are uncorrelated or almost uncorrelated with $Y$. For example, if $\EE{XY} = 0$, this definition would require $f$ to also be perfectly uncorrelated with $Y$.} However, $f$ will not be $(\varepsilon, f)$-accurate, because $\EE{(Y - f)f}^2 \approx \EE{f^2}^2 \gg \varepsilon^2 \EE{f^2}^2$.
\end{remark}

\subsubsection{Approximate merges of pairing structure estimates}
For this section, let $n$ be a positive even integer and let $\mathcal{D}_n$ be the distribution of hafnians of matrices $\Sigma$ with off-diagonal entries sampled from $\mathcal{N}(0, 1)$, as defined earlier. For convenience, we will let $Y$ be the value of the hafnian (that is, $Y$ is the random variable that is distributed according to $\mathcal{D}_n$).

In the previous section we showed that finding \emph{exact} coefficients for linearly merging estimates of $Y$ of the form $X_{\mathcal{T}}$, where $\mathcal{T}$ is a pairing structure, is not computationally tractable. In this section, we discuss \emph{approximate} linear merges.

Let $\mathcal{T}_1, \dots, \mathcal{T}_m$ be pairing structures. We are interested in an estimator $f$ of $Y$ that is $(\delta, \varepsilon, \{1, X_{\mathcal{T}_1}, \dots, X_{\mathcal{T}_m}, f\})$-multiaccurate over $\mathcal{D}_n$. Further, we would like $f$ to be computable in polynomial time in $m$, $n$, $1/\varepsilon$, $1/\delta$, and the representation sizes of $\mathcal{T}_1, \dots, \mathcal{T}_m$. We exhibit an $f$ that satisfies these properties, so long as the correlation matrix of $X_{\mathcal{T}_1}, \dots, X_{\mathcal{T}_m}$ is well-conditioned. (The \emph{correlation} of two random variables $Z_1$ and $Z_2$ is defined as $\frac{\Cov{Z_1, Z_2}}{\sqrt{\Var{Z_1} \Var{Z_2}}}$.)

Our algorithm is fairly straightforward: recall from Theorem~\ref{thm:no_exact_merge} the main obstacle to accurately merging estimates via linear regression: estimating the covariances $\EE[\mathcal{D}_n]{X_{\mathcal{T}_i} X_{\mathcal{T}_j}}$ is computationally intractable. On the other hand, we can approximate these covariances fairly accurately. Concretely, we have that $\EE{X_{\mathcal{T}_i} X_{\mathcal{T}_j}} = \abs{S(\mathcal{T}_i) \cap S(\mathcal{T}_j)}$. We can estimate the size of this intersection by sampling random pairings from $S(\mathcal{T}_i)$ and observing the fraction that also belong to $S(\mathcal{T}_j)$. The number of samples taken by the algorithm is determined adaptively, and is related to the condition number of the correlation matrix. We then do OLS regression using this approximate covariance matrix. The resulting estimator is approximately multiaccurate.

We state our estimator formally as Algorithm~\ref{alg:approx_acc_haf} in the appendix. The estimator is given by \textsc{Estimator}, while the functions \textsc{NumPairings}, \textsc{Sample}, \textsc{Split}, and \textsc{Contains} are helper functions. Our main result concerning our algorithm is Theorem~\ref{thm:haf_approx}:

\begin{restatable}{theorem}{hafapprox} \label{thm:haf_approx}
    Let $\mathcal{T}_1, \dots, \mathcal{T}_m$ be pairing structures with index set $[n]$. Let $\sigma_m$ be the smallest singular value of the correlation matrix of $X_{\mathcal{T}_1}, \dots, X_{\mathcal{T}_m}$. Let $\delta, \varepsilon > 0$ and let $f$ be the output of $\textsc{Estimator}$ (with arguments $\mathcal{T}_1, \dots, \mathcal{T}_m, \delta, \varepsilon$). Then $\textsc{Estimator}$ is $(\delta, \varepsilon, \{1, X_{\mathcal{T}_1}, \dots, X_{\mathcal{T}_m}, f\})$-multiaccurate, and the expected runtime of $\textsc{Estimator}$ is polynomial in $1/\delta$, $1/\varepsilon$, $1/\sigma_m$, and the sum of the representation sizes of $\mathcal{T}_1, \dots, \mathcal{T}_m$.
\end{restatable}

We provide the proof of Theorem~\ref{thm:haf_approx} in the appendix. The proof proceeds in three steps:
\begin{enumerate}[label=(\arabic*)]
    \item We consider the difference $\Delta$ between our estimate $\hat{C}$ of the correlation matrix of the $X_{\mathcal{T}_i}$'s and the actual correlation matrix $C$. If we could estimate $C$ perfectly (i.e.\ $\Delta = 0$) then the estimate given by \textsc{Estimator} would be exactly $\{1, X_{\mathcal{T}_1}, \dots, X_{\mathcal{T}_m}, f\}$-multiaccurate. What if $\Delta$ is merely close to $0$? We show that, so long as the spectral norm $\norm{\Delta}_2$ is sufficiently small (namely, if $\norm{\Delta}_2 \le \frac{\varepsilon \sigma_m \hat{\sigma}_m}{2m}$, where $\hat{\sigma}_m$ is the smallest singular value of $\hat{C}$), then $f$ is approximately multiaccurate (Lemma~\ref{lem:if_delta_small}).
    
    \item We show that $\norm{\Delta}_2$ is indeed this small with high probability (Lemma~\ref{lem:delta_small_whp}). We do so by bounding $\norm{\Delta}_2$ in terms of the sum of the squared errors of our estimates of each entry of $C$, and using concentration inequalities in order to bound the probability that this sum of errors is large.

    \item We show that the expected number of samples taken by \textsc{Estimator} is polynomial in $m, 1/\delta, 1/\varepsilon$, and $1/\sigma_m$: in particular, that it is $O \parens{\frac{m^2(m^2 + \ln(1/\delta))}{\varepsilon^2 \sigma_m^4}}$.
\end{enumerate}

The main limitation of Theorem~\ref{thm:haf_approx} is that the runtime of \textsc{Estimator} depends on $\sigma_m$ (specifically, it goes as $1/\sigma_m^4$). This is unfortunate because $\sigma_m$ is unknown (and hard to approximate) and can be very small: for example, if $m = 2$ and $S(\mathcal{T}_1)$ and $S(\mathcal{T}_2)$ are both exponentially large and differ by only one pairing, then $\sigma_m$ is exponentially small.

Is it possible to produce an estimation algorithm whose output $f$ is $(\delta, \varepsilon, \{1, X_{\mathcal{T}_1}, \dots, X_{\mathcal{T}_m}, f\})$-multiaccurate and whose expected runtime does not depend on $\sigma_m$? We believe that using a number of samples that does not depend on $\hat{\sigma}_m$ to estimate $C$, and then using ridge regression \citep{hastie09elements} in place of linear regression, will yield  a $(\delta, \varepsilon, \{1, X_{\mathcal{T}_1}, \dots, X_{\mathcal{T}_m}\})$-multiaccurate estimator in polynomial time that does not depend on $\sigma_m$. However, we believe that such an estimator would not necessarily be self-accurate.\footnote{In the context of \emph{exact} accuracy, an estimator that is multiaccurate with respect to a set of predictors will also be accurate with respect to any linear combination. However, the picture is more complicated for approximate accuracy. Exact merges may require regression coefficients on the predictors $X_{\mathcal{T}_i}$ that are exponentially large in $m$; as $\varepsilon \to 0$, regression coefficients for an approximate merge need to become large as well. Because of these large coefficients, approximate accuracy with respect to each predictor does not necessarily guarantee approximate self-accuracy.} (See Remark~\ref{rem:self_accuracy} for why we insist that our estimators be approximately self-accurate.) Thus, this question remains open.

There are other natural follow-up questions to ask. For example, we have defined the \emph{union} of two pairing structures as a merge of two sets of pairings that are promised to be disjoint. However, often we wish to merge two non-disjoint sets of pairings. We can do so by allowing pairing structures to represent not just sets of pairings, but formal linear combinations of pairings. Then, we can replace the union operation with a \emph{linear merge} operation. Specifically, given a pairing structure $\mathcal{T}_1$ that represents the linear combination $\sum_{p \in \mathcal{P}_2(K)} c_{1, p} p$ and a pairing structure $\mathcal{T}_2$ that represents the linear combination $\sum_{p \in \mathcal{P}_2(K)} c_{2, p} p$, we let the linear merge of $\mathcal{T}_1$ and $\mathcal{T}_2$ represent the linear combination of $\mathcal{T}_1$ and $\mathcal{T}_2$ that best approximates $\sum_{p \in \mathcal{P}_2(K)} p$. At present, the question of approximate merging of arguments based on this generalized notion of pairing structures remains open.

\section{Estimating the permanent of a matrix} \label{sec:perm}
We now turn our attention to a different but closely related setting: estimating permanents of matrices. The permanent of an $n \times n$ matrix $A$ is defined as
\[\perm(A) := \sum_{\sigma\in S_n} \prod_{i = 1}^n A_{i, \sigma(i)},\]
where $S_n$ is the group of permutations of $[n]$.

The permanent is like the determinant, except that it does not take into account the signs of the permutations. However, unlike the determinant, the permanent of a matrix cannot be computed efficiently: computing the permanent of a matrix with $0, 1$ entries is $\mathsf{\#P}$-hard \citep{valiant79permanent}. In fact, it is NP-hard to approximate the permanent of a positive semidefinite matrix even to within an exponential factor \citep{ebrahimnejad24permanent}.\footnote{Specifically, it is NP-hard to approximate the permanent of an $n \times n$ PSD matrix to within a factor of $e^{0.577n}$. Positive semidefinite matrices have nonnegative permanents, so approximation ratios are reasonable in this context.} On the other hand, we will see a variety of heuristic arguments about the permanent of a matrix; incorporating all of these arguments into a single estimate of the permanent presents a challenge. For these reasons, permanent estimation is a good testbed for heuristically estimating quantities from multiple arguments.

We will begin by introducing three heuristic estimates for the permanent, which we call the \emph{row sum,} \emph{column sum,} and \emph{matrix sum} estimates. Although there is a straightforward accurate merge of these estimate (via linear regression), this merged estimator gives negative estimates for the permanent of some matrices with only positive entries. Although an \emph{ad hoc} modification of the linear regression estimator may fix this issue, we see no principled way to correct the estimator in light of this observation. By contrast, there \emph{is} a well-motivated estimator that merges the row sum, column sum, and matrix sum estimates, but the estimator is multiplicative in nature and does not satisfy any accuracy properties (which are by their nature additive). Our exploration gives us further evidence that accuracy is not the correct approach to formalizing the principle of unpredictable errors for heuristic estimators.\\

For the remainder of this section, we will fix a positive integer $n$ and define $\mathcal{D}$ to be the distribution of matrices whose entries are independently sampled from the standard normal distribution $\mathcal{N}(0, 1)$. We will abuse notation by also letting $\mathcal{D}$ represent the distribution over \emph{permanents} of these matrices; the meaning will always be inferrable from context.

\subsection{Simple estimators for the permanent}
We begin by introducing three different estimates for the permanent of a matrix, all of which are inspired by the presumption of independence \citep{christiano2022formalizing}.\footnote{Code for computing various estimates of permanents and hafnians, and for estimating OLS regression coefficients over various distributions, can be found at: \texttt{\url{https://github.com/alignment-research-center/heuristic-hafnian}}.} Concretely, we may write
\[\perm(A) = n! \EE[\sigma \sim S_n]{\prod_{i = 1}^n A_{i, \sigma(i)}},\]
where $\sigma$ is uniformly sampled from $S_n$. If we treat $\sigma(1), \dots, \sigma(n)$ as independent random variables, then we may exchange the expectation and the product. This gives us the \emph{row sum estimate} for the permanent of $A$:
\[E_{\row}(A) := n! \prod_{i = 1}^n \EE[\sigma \sim S_n]{A_{i, \sigma(i)}} = \frac{n!}{n^n} \prod_{i = 1}^n \sum_{j = 1}^n A_{i, j}.\]
Note that the row sum estimate is the average permanent of all matrices obtained from $A$ by shuffling the entries in each row of $A$ independently.

If instead we write $\perm(A) = n! \EE[\tau \sim S_n]{\prod_i A_{\tau(i), i}}$ and presume that $\tau(1), \dots, \tau(n)$ are independent, we obtain the \emph{column sum estimate} for the permanent of $A$:
\[E_{\col}(A) := \frac{n!}{n^n} \prod_{j = 1}^n \sum_{i = 1}^n A_{i, j}.\]
Finally, we may also write
\[\perm(A) = n! \EE[\tau, \sigma \sim S_n]{\prod_{i = 1}^n A_{\tau(i), \sigma(i)}},\]
where $\tau, \sigma \in S_n$ are sampled uniformly and independently. If we presume that all of $\sigma(1), \dots, \sigma(n), \tau(1), \dots, \tau(n)$ are independent, we obtain the \emph{matrix sum estimate} for the permanent of $A$:
\[E_{\ms}(A) := n! \prod_{i = 1}^n \EE[\tau, \sigma \sim S_n]{A_{\tau(i), \sigma(i)}} = \frac{n!}{n^{2n}} \parens{\sum_{i = 1}^n \sum_{j = 1}^n A_{i, j}}^n.\]

What properties do these estimators have? $E_{\row}$ and $E_{\col}$ are $1$-accurate and self-accurate over $\mathcal{D}$, and are both $E_{\ms}$-accurate as well.

\begin{claim}
    $E_{\row}$ is $\{1, E_{\row}, E_{\ms}\}$-multiaccurate over $\mathcal{D}$. $E_{\col}$ is $\{1, E_{\col}, E_{\ms}\}$-multiaccurate over $\mathcal{D}$.
\end{claim}

\begin{proof}
    To see that $E_{\row}$ is $1$-accurate over $\mathcal{D}$, note that the average permanent over $\mathcal{D}$ is zero (consider flipping the sign of all entries in the first row of a matrix), as is the average value of $E_{\row}$ over $\mathcal{D}$ (for the same reason).
    
    To show that $E_{\row}$ is $E_{\row}$-accurate and $E_{\ms}$-accurate, we first observe that $E_{\row}(A)$ is the best estimate of $\perm(A)$ given only the set of $n$ entries in each row (but not their order). To see this, consider any matrix $A$, and let $\mathcal{D}_{\row}(A)$ be $\mathcal{D}$ restricted to matrices with the same set of entries in each row as $A$. Then $\mathcal{D}_{\row}(A)$ is the uniform distribution over the $(n!)^n$ matrices obtained by permuting the entries within each row of $A$; the average permanent over $\mathcal{D}_{\row}(A)$ is equal to $E_{\row}(A)$.
    
    Now, let $X$ be any predictor (i.e.\ function of the matrix $A$) that depends only on the set of entries in each row. Both $E_{\row}$ and $E_{\ms}$ are examples of such predictors. For notational convenience, let $\mathcal{R}(A)$ be the $n$-tuple whose $i$-th entry is the set of $n$ entries in the $i$-th row of $A$. We have
    \begin{align*}
        \EE[A \sim \mathcal{D}]{(\perm(A) - E_{\row}(A))X} &= \EE[A \sim \mathcal{D}]{\EE{(\perm(A) - E_{\row}(A))X \mid \mathcal{R}(A)}}\\
        &= \EE[A \sim \mathcal{D}]{X \EE{\perm(A) - E_{\row}(A) \mid \mathcal{R}(A)}} = 0,
    \end{align*}
    where the second step follows by the ``taking out constants'' property of conditional expectations. Thus, $E_{\row}$ is $E_{\row}$-accurate and $E_{\ms}$-accurate. The fact that $E_{\col}$ is $\{1, E_{\col}, E_{\ms}\}$-multiaccurate over $\mathcal{D}$ is analogous.
\end{proof}

What about $E_{\ms}$? Unfortunately, $E_{\ms}$ is only $1$-accurate for odd $n$ (for even $n$, $E_{\ms}$ is always nonnegative), and is not self-accurate for $n > 2$. On the other hand, Proposition~\ref{prop:multiacc} tells us that the OLS estimator of $Y$ with predictors $1$ and $E_{\ms}$ will be $1$-accurate and self-accurate. It can be shown that this estimator is
\begin{equation} \label{eq:msprime}
    E'_{\ms}(A) := \frac{n!}{(2n - 1)!! - \mathbbm{1}_{\text{$n$ is even}} (n - 1)!!^2} \parens{E_{\ms}(A) - \frac{\mathbbm{1}_{\text{$n$ is even}} n! (n - 1)!!}{n^n}}.
\end{equation}

Similarly, we may combine $E_{\row}$ and $E_{\col}$ via OLS to create an estimator that is $\{1, E_{\row}, E_{\col}\}$-multiaccurate:
\[E_{\row, \col}(A) := \frac{n^n}{n^n + n!}(E_{\row}(A) + E_{\col}(A)).\]
On the other hand, $E_{\mathrm{row}, \mathrm{col}}$ is not $E_{\ms}$-accurate (even though it is $E_{\row}$-accurate and $E_{\row}$ is $E_{\ms}$-accurate -- see Remark~\ref{remark:nontransitive}). We can make the estimate better by adding a constant multiple of $E_{\ms}$:
\begin{equation} \label{eq:rowcolms}
    E_{\row, \col, \ms}(A) := \frac{\parens{1 - \frac{n!}{n^n}}b(E_{\row}(A) + E_{\col}(A)) - \parens{1 - \frac{n!}{n^n}} E'_{\ms}(A)}{2 \parens{1 - \frac{n!}{n^n}b} - \parens{1 - \frac{n!}{n^n}}},
\end{equation}
where $b := \frac{n!}{(2n - 1)!! - \mathbbm{1}_{\text{$n$ is even}} (n - 1)!!^2}$ is the regression coefficient of the permanent onto $E_{\ms}$, as in equation~\ref{eq:msprime}.

\subsection{Adding in non-negativity}
If we are interested in heuristically estimating the permanent of a matrix drawn from $\mathcal{D}$, then we may want to use a heuristic estimator that can incorporate the aforementioned estimates $E_{\row}$, $E_{\col}$, and $E_{\ms}$. More precisely, we may want our heuristic estimator $\GG$ to be self-accurate and $\{1, E_{\row}, E_{\col}, E_{\ms}\}$-multiaccurate. This alone is simple enough: for example, there could be arguments $\pi_{\row}$, $\pi_{\col}$, and $\pi_{\ms}$ that cause $\GG$ to output $E_{\row}$, $E_{\col}$, and $E_{\ms}$, respectively, and $\GG$ may merge these arguments with OLS regression as discussed in Section~\ref{sec:g_multiacc_ols}. For example, $\GG(\perm(A) \mid \pi_{\row}, \pi_{\col}, \pi_{\ms}) = E_{\row, \col, \ms}(A)$ as in equation~\ref{eq:rowcolms}.

However, we may want $\GG$ to satisfy some additional properties. A simple one: there should be an argument that causes $\GG$ to output nonnegative estimates for permanents of matrices with nonnegative entries. That is, there should be an argument $\pi_{\ge 0}$ such that $\GG(\perm(A) \mid \Pi) \ge 0$ for all $\Pi \ni \pi_{\ge 0}$ and all matrices $A \ge 0$. (We will write $A \ge 0$ to mean that all entries of $A$ are nonnegative.)

The aforementioned OLS regression estimator does not satisfy this property by default: $E_{\row, \col, \ms}(A)$ can be negative, even if $A \ge 0$.\footnote{For example, for $n$ odd, $A$ could be the $n \times n$ matrix with all zeros except for one $1$; then $E_{\row}(A) = E_{\col}(A) = 0$, while $E_{\ms}(A)$ is positive.} So how do we construct a heuristic estimator $\GG$ that is self-accurate, $\{1, E_{\row}, E_{\col}, E_{\ms}\}$-multiaccurate, and is able to estimate matrices with nonnegative entries as having nonnegative permanents?

We do not know of a natural way of doing so. It is perhaps possible to satisfy all of these constraints in an artificial way: for example, there may be coefficients $c \ge 0$ and $\alpha, \beta_{\row}, \beta_{\col}, \beta_{\ms}$ for which the estimator
\begin{equation} \label{eq:unnatural}
    \GG(\perm(A) \mid \pi_{\row}, \pi_{\col}, \pi_{\ms}, \pi_{\ge 0}) := \begin{cases}c & \text{if $A \ge 0$} \\ \alpha + \beta_{\row} E_{\row} + \beta_{\col} E_{\col} + \beta_{\ms} E_{\ms} & \text{otherwise}\end{cases}
\end{equation}
satisfies all of the aforementioned properties. However, it seems like a mistake for $\GG$ to entirely ignore the estimates given by $\pi_\row$, $\pi_\col$, and $\pi_\ms$ in the case $A \ge 0$. To the extent that we are interested in finding a heuristic estimator that merges estimates in a natural, well-motivated way, it seems that we ought to reject a merge such as the one in equation~\ref{eq:unnatural}.

There is, however, a quite natural merge of $\pi_\row$, $\pi_\col$, $\pi_\ms$ for matrices with non-negative entries: specifically, $\frac{E_\row(A) E_\col(A)}{E_\ms(A)}$. One intuition for this estimate is as follows: we start with the estimate $E_\ms(A)$, which computes the average product of $n$ randomly selected entries of $A$ (and multiplies by $n!$, since the permanent is a sum of $n!$ products). Although we justified $E_\ms$ earlier with a presumption of independence, this estimator fails to notice that all entries must be in distinct rows and distinct columns in order to count toward the permanent. By contrast, $E_\row$ computes the average product of $n$ randomly selected entries of $A$, one from each row (and multiplies by $n!$). Thus, $E_\row$ -- unlike $E_\ms$ -- notices that entries must come from different rows; correspondingly, multiplying our initial estimate of $E_\ms(A)$ by a corrective factor of $\frac{E_\row(A)}{E_\ms(A)}$ changes our estimate to take this into consideration. Unsurprisingly, this gives us the estimate $E_\row(A)$. We then apply the analogous corrective factor for the observation that entries must come from different columns; this means multiplying by $\frac{E_\col(A)}{E_\ms(A)}$, thus giving us the aforementioned estimator $\frac{E_\row(A) E_\col(A)}{E_\ms(A)}$.

More formally, we may justify this estimator in terms of a presumption of independence, much as we did for $E_\row$, $E_\col$, and $E_\ms$ above.\footnote{See \citet[\S A.2.3]{christiano2022formalizing} for a closely related heuristic estimate about Hamiltonian cycles.} Specifically, let $A$ be an $n \times n$ matrix with $0, 1$ entries that are not all zero, and let $(I_1, J_1), \dots, (I_n, J_n)$ be sampled uniformly and independently from $\{1, \dots, n\}^2$. Let $S := \prod_{k = 1}^n A_{I_k, J_k}$, $R := S \cdot \mathbbm{1}_{I_1, \dots, I_k \text{ are distinct}}$, and $C := S \cdot \mathbbm{1}_{J_1, \dots, J_k \text{ are distinct}}$. We observe the following relationships between $R, C, S$ and $E_\row$, $E_\col$, and $E_\ms$:

\begin{align*}
    \EE[I, J]{S^2} = \EE[I, J]{S} &= \frac{1}{n!} E_{\ms}(A)\\
    \EE[I, J]{RS} = \EE[I, J]{R^2} = \EE[I, J]{R} &= \frac{1}{n^n} E_{\row}(A)\\
    \EE[I,J]{CS} = \EE[I, J]{C^2} = \EE[I, J]{C} &= \frac{1}{n^n} E_{\col}(A)\\
    \EE[I, J]{RC} &= \frac{n!}{n^{2n}} \perm(A).
\end{align*}

Since $\perm(A) = \frac{n^{2n}}{n!} \EE[I, J]{RC}$, any estimate of $\EE[I, J]{RC}$ yields a corresponding estimate of the permanent. To estimate $\EE[I, J]{RC}$, we replace $R$ and $C$ by their linear regressions onto $S$. This can be thought of as presuming that the residuals are independent. Doing so gives the estimate
\[\EE[I, J]{RC} \approx \EE[I, J]{R} \EE[I, J]{C} + \frac{\Cov[I, J]{R, S} \Cov[I, J]{C, S}}{\Var[I, J]{S}}.\]
Substituting our above expressions for $\EE[I, J]{RS}$ and so forth in terms of $E_{\ms}$, $E_{\row}$, and $E_{\col}$ gives us the estimate
\[\perm(A) \approx \frac{E_\row(A) E_\col(A)}{E_\ms(A)},\]
as promised.

This estimate appears to combine $E_\row$, $E_\col$, and $E_\ms$ for 0-1 matrices in a reasonable way, and additionally has the property that it is non-negative for matrices with non-negative entries. Unfortunately, however, it is not $1$-accurate (over 0-1 matrices, or indeed -- as far as we know -- over any natural distribution over matrices), nor is it self-accurate. Further, it is not well-defined over $\mathcal{D}$, or indeed over any distribution containing matrices for which $E_\ms$ can be zero while $E_\row$ and $E_\col$ are not zero. Changing the estimator in equation~\ref{eq:unnatural} by returning $\frac{E_\row(A) E_\col(A)}{E_\ms(A)}$ if $A \ge 0$ still does not yield a $1$-accurate or self-accurate estimator over $\mathcal{D}$.

One way to improve the accuracy properties of this estimator over 0-1 matrices is to instead use the \emph{unique sum estimate}
\[E_{\mathrm{us}}(A) := \frac{n!}{\binom{n^2}{n}} \sum_{\substack{S \subseteq [n]^2 \\ \abs{S} = n}} \prod_{(i, j) \in S}^n A_{i, j}\]
in place of the matrix sum estimate $E_\ms$.\footnote{While the unique sum estimate cannot in general be computed efficiently, it can be computed efficiently for 0-1 matrices.} The estimate $\frac{E_\row(A) E_\col(A)}{E_{\mathrm{us}}(A)}$ can be justified in a similar way to $\frac{E_\row(A) E_\col(A)}{E_{\ms}(A)}$ by considering the quantity $U := S \mathbbm{1}_{(I_1, J_1), \dots, (I_n, J_n) \text{ are distinct}}$. Although this new estimator closer to being self-accurate than $\frac{E_\row(A) E_\col(A)}{E_\ms(A)}$, it is still not perfectly self-accurate.

\section{Revisiting desiderata for heuristic estimation} \label{sec:revisiting}
In the previous sections, we explored accuracy as a potential formal property of heuristic estimators. In this section, we put our discussion in the broader context of other desiderata for heuristic estimation.

\subsection{Iterated estimation and error orthogonality, revisited}
In Section~\ref{sec:iterated_estimation}, we defined the iterated estimation and error orthogonality properties of $\GG$ to formalize the principle of unpredictable errors. We contextualized these properties in terms of subjective expected values from Bayesian epistemology, thinking of $\GG$ as an algorithm that extracts a subjective expectation from a state of knowledge.

In Section~\ref{sec:challenges}, we discussed difficulties of these two properties. The difficulties stem from the fact that both properties involve $\GG$ estimating its own output. In Section~\ref{sec:multiacc}, we attempted to circumvent these difficulties by instead requiring $\GG$ to be (objectively) accurate over a distribution of input expressions. However, in Sections~\ref{sec:hafnian} and~\ref{sec:perm}, we demonstrated some barriers to this objective approach.

\emph{First,} we introduced a quite simple and natural estimation problem: estimating the expected product of $n$ jointly normal random variables. We observed that this problem was equivalent to estimating the hafnian of the covariance matrix, i.e.\ the sum of products of covariances over all possible ways of pairing the variables. This observation gave us a natural class of estimators, namely the sum over any subset of the set of all pairings. We asked whether it was possible to produce an accurate merge of two such estimators over the distribution of covariance matrices in which each off-diagonal entry is drawn from $\mathcal{N}(0, 1)$. While we have not definitively ruled out this possibility, we showed that producing an accurate \emph{linear} merge is $\mathsf{\#P}$-hard. Thus, perfect accuracy seems likely to be out of reach even in quite simple and natural settings.

\emph{Second,} we introduced a notion of approximate accuracy and exhibited a sampling-based algorithm for merging the aforementioned estimators into an approximately accurate estimator. The algorithm estimates correlations between the estimators and then performs OLS regression. However, the expected runtime of the algorithm has polynomial dependence on the condition number of the correlation matrix of the estimators.\footnote{More precisely, we showed that the expected runtime of the algorithm depends polynomially on the reciprocal of the smallest singular value of the covariance matrix. Since the largest singular value of the correlation matrix is between $1$ and $m$, this polynomial dependence also implies a polynomial dependence on the condition number of the matrix.} This is an issue, given that the condition number can be exponential in $n$. Substituting ridge regression for OLS regression removes this polynomial dependence; however, we believe that this comes at the expense of self-accuracy for the merged estimator. We are not aware of an algorithm for merging these estimators in a way that is, with high probability, approximately accurate with respect to the estimators being merged and also approximately self-accurate. While we have not ruled out such an algorithm, we believe that such an algorithm is unlikely to be simple, despite the ostensibly simple problem setting.

\emph{Third,} we explored another natural estimation problem: estimating the permanent of an $n \times n$ matrix. We introduced three estimators -- $E_\row$, $E_\col$, and $E_\ms$ -- and justified the estimators based on three different presumptions of independence. We observed that although there is an accurate merge of the estimators\footnote{Over the distribution $\mathcal{D}$ of matrices whose entries are independently drawn from $\mathcal{N}(0, 1)$.} (namely, the linear regression merge), modifying this merge to satisfy other desirable properties in seemingly out of reach. For example, there seems to be no natural way to modify this linear regression estimator to give non-negative estimates for permanents of non-negative matrices while still maintaining accuracy properties. In fact, we showed that there \emph{is} a natural merge of $E_\row$, $E_\col$, and $E_\ms$ for non-negative matrices -- namely, $\frac{E_\row E_\col}{E_\ms}$ -- but this merge does not satisfy our accuracy properties. The fact that our notion of accuracy does not play well with ``multiplicative'' estimates such as this one appears to be an important limitation.

\emph{Finally,} perhaps the greatest limitation of the accuracy perspective on the principle of unpredictable errors is that the choice of distribution $\mathcal{D}$ is seemingly arbitrary. We would like $\GG$ to produce subjective estimates of quantities based on a state of knowledge; while assuming a background distribution may be useful as a crutch, our ultimate goal is to find a ``distribution-free'' perspective.

In light of these considerations, we believe that accuracy is not an adequate substitute for iterated estimation and error orthogonality. We still believe that a good heuristic estimator $\GG$ -- if one exists -- ought to satisfy some formal property that formalizes the principle of unpredictable errors. However, we do not currently have a suggestion for such a property that we feel satisfied with.

\subsection{Other desiderata}
In Section~\ref{sec:new_properties}, we recalled the \emph{linearity} and \emph{respect for proofs} desiderata introduced by \citet{christiano2022formalizing}. Roughly speaking, linearity requires that $\GG(X + Y \mid \Pi) = G(X \mid \Pi) + G(Y \mid \Pi)$, while respect for proofs requires that if $\pi$ proves that $Y \ge 0$, then $G(Y \mid \Pi) \ge 0$ for all $\Pi$ containing $\pi$.

\citet[Chapter 9]{neyman2024algorithmic} explored these two desiderata in the context of estimating the expected output (i.e.\ acceptance probability) of boolean circuits. In that setting, it is possible to satisfy linearity and respect for proofs using a linear programming-based algorithm. In brief, the algorithm creates a list of constraints that the estimator's outputs must satisfy in order for the estimator to satisfy linearity and respect for proofs. Then, the algorithm finds a solution to these constraints using linear programming. Unfortunately, while this algorithm satisfies linearity and respect for proofs, it has many drawbacks. For example, the algorithm's estimates are affected by the introduction of uninformative arguments.

\citet[Chapter 9]{neyman2024algorithmic} also discusses desiderata beyond linearity and respect for proofs, including iterated estimation and \emph{pulling out known factors}: for all $X, Y$ and $\Pi' \subseteq \Pi$, we have
\[\GG(\GG(X \mid \Pi') Y \mid \Pi) = \GG(X \mid \Pi') G(Y \mid \Pi).\]
Much like iterated estimation and error orthogonality, this property is inspired by conditional expectations: specifically the \emph{pulling out known factors} property, which states that for $\sigma$-algebras $\mathcal{H}' \subseteq \mathcal{H}$, we have $\EE{\EE{X \mid \mathcal{H}'} Y \mid \mathcal{H}} = \EE{X \mid \mathcal{H}'} \EE{Y \mid \mathcal{H}}$.\footnote{Typically, the property is instead stated as follows: for random variables $X, Y$ where $X$ is $\mathcal{H}$-measurable, we have $\EE{XY \mid \mathcal{H}} = X \EE{Y \mid \mathcal{H}}$. However, our statement makes clearer the connection to the corresponding desideratum for $\GG$.} However, this desideratum likely suffers from the same difficulties as iterated estimation and error orthogonality, which arise from the fact that the properties refer to $\GG$'s estimates of its own outputs.

We are left with an interesting formalization challenge. Linearity and respect for proofs, while compelling, are insufficiently stringent: a heuristic estimator that satisfies them may be unnatural or suffer from serious drawbacks, and may not capture any reasonable process of forming subjective expectations. This raises the question of whether we can state an additional formal property that will help pin down the behavior of $\GG$. Our various attempts so far -- iterated estimation, error orthogonality, accuracy, and pulling out known factors -- all seem inadequate in their current forms. However, we find it plausible that there exists an additional formal property that, together with linearity and respect for proofs, will successfully pin down the behavior of $\GG$: be satisfiable, but only by algorithms that reasonably capture the notion of subjective expectation. Finding such a property, and searching for a heuristic estimator that satisfies it, is an interesting and important direction for future work.

\section{Future directions} \label{sec:conclusion}
\subsection{The main challenges ahead for heuristic estimation}
In this work, we studied heuristic estimation by exploring ways to formalize the principle of unpredictable errors. Through this exploration, we saw a different perspective on this principle: that a heuristic estimator ought to ``merge'' arguments in a way that takes into account the information provided by each argument being merged. We attempted to formalize this property from two different angles: subjective Bayesian expectation (Section~\ref{sec:iterated_estimation}) and objective expected value over a distribution (Sections~\ref{sec:multiacc} through~\ref{sec:perm}). In both cases, we ran into substantial roadblocks, suggesting that we have not yet found an adequate formalization.

Ultimately, the purpose of formalizing the principle of unpredictable errors is to help guide the search for a heuristic estimator $\GG$ that merges heuristic arguments in a reasonable way. Finding such a $\GG$ remains an important and difficult open problem.

Handling argument merges is not the only obstacle to finding a reasonable $\GG$. Another challenge is handling cherry-picked arguments, i.e.\ arguments that have been optimized to cause $\GG$'s estimate to be as small or as large as possible (perhaps by a computationally bounded adversary). For example, suppose that $Y$ is a sum of many mathematical expressions, and $\pi$ computes all of the positive terms but none of the negative terms; what should $\GG(Y \mid \pi)$ return? Perhaps $\GG$ simply returns the sum of the positive terms; or perhaps $\GG$ must be responsible for noticing that $\pi$ is cherry-picked. In the former case, we would like to find a notion of ``quality'' for heuristic arguments that would consider such a $\pi$ to be low-quality. One potential path forward is \emph{surprise accounting} \citep{hilton24surprise}: penalizing arguments for each computation presented in the argument. While surprise accounting would not prevent $\GG$ from being fooled by cherry-picked arguments, it would help \emph{search} for high-quality arguments to give to $\GG$ as input, by benefiting short arguments that explain the properties of the expression $Y$ over arguments that exhaustively compute $Y$ (or part of $Y$).

\subsection{Potential application: Understanding neural networks}
In this section, we discuss our main motivation for studying heuristic estimation: a potential application to understanding the behavior of neural networks.

Until recent years, our understanding of neural network behavior has been limited to what we can observe empirically based on input-output behavior. For example, one can try to understand the adversarial robustness properties of an image model by creating input images that are optimized for the model mislabeling the image \citep{goodfellow15adversarial}.

However, this approach leaves important gaps in our understanding. First, conclusions drawn about the behavior of a neural network on one input distribution may not apply to a different distribution. Second, some properties may not be easily measurable: for example, measuring the truthfulness of a large language model can only be done to the extent that we can distinguish true responses from false ones. Third, even easily measurable events cannot always be detected through sampling: for example, if a neural network exhibits a rare but very undesirable behavior with probability $p$, then estimating the frequency of the behavior via sampling may require $\Omega(1/p)$ samples.

More recently, the field of neural network interpretability has aimed to fill these gaps by understanding neural networks' representations of various concepts. For example, \citet{bricken23sae} demonstrate that feature representations in a one-layer transformer can be learned by taking advantage of the fact that features are sparse (on a typical input, only a small fraction of features are active). However, existing interpretability techniques only work under strong assumptions about how neural networks represent information (such as the linear representation hypothesis, see e.g.\ \citet{park23linear}). Further, the goal of existing interpretability techniques is to find human-understandable representations of model internals; this can only work insofar as concepts represented by neural networks can in theory be understood by humans.

Some researchers have instead pursued formal verification: that is, formally proving properties of neural networks. For example, \citet{raghunathan18certified} formally proved adversarial robustness guarantees for neural networks with one hidden layer. More recently, \citet{gross24proofs} used computer-assisted proofs to prove lower bounds on the accuracy of small transformers on algorithmic tasks. Formal verification does not rely on human understanding; this means that formal verification techniques can (in theory) be used to verify properties of neural networks whose explanations are not human-understandable. On the other hand, proving strong guarantees about interesting properties of large neural networks (such as out-of-distribution robustness) seems out of reach: it is possible that no compact proof exists -- or, even if one does, it is not clear whether it can be found.\\

By contrast, \emph{heuristic arguments} about properties of neural networks may have the important advantages of both formal verification and classical interpretability approaches. On the one hand, heuristic arguments (much like proofs) are formal objects that are not required to be human-understandable (even if a human could check each individual deduction). This means that heuristic arguments could be used to reason about properties of neural networks for which no compact human-understandable explanation exists. On the other hand, heuristic arguments (much like classical interpretability approaches) do not require perfect certainty to be considered valid. This allows for short heuristic arguments of complex properties of large models, even when no short proofs of those properties exist.

Below, we give three examples of problems involving understanding neural network behavior. We believe that none of these problems can be solved with existing approaches, but that all of them have the potential to be solved using heuristic arguments.

\paragraph{Mechanistic anomaly detection.} Let $M$ be a neural network that was trained on a distribution $\mathcal{D}$ of inputs $x$ using the loss function $L(x, M(x))$.\footnote{For example, $L$ could be based on a trained reward predictor, as in reinforcement learning from human feedback \citep{christiano17rlhf}.} Suppose that $M$ successfully learns to achieve low loss: that is, $\EE[x \sim \mathcal{D}]{L(x, M(x))}$ is small. Let $x^*$ be a (perhaps out-of-distribution) input. We call $x^*$ a \emph{mechanistic anomaly} for $M$ if $M$ gets a low loss on $x^*$, but for a ``different reason'' than the reason why it gets low average loss on $\mathcal{D}$. In other words, mechanistic anomalies are inputs on which $M$ acts in a \emph{seemingly} reasonable way, but via anomalous internal mechanisms.\footnote{For example, if $M$ is a financial assistant that takes actions such as buying stocks and transferring money between bank accounts, then $M$ might have a low loss on $\mathcal{D}$ because it makes good financial decisions, but a low loss on $x^*$ because it implements a money laundering scheme that $L$ fails to notice \citep[Chapter 9, \S5]{neyman2024algorithmic}.} To detect a mechanistic anomaly, reasoning about $M$'s internal structure may be necessary.

How could we use a heuristic estimator to detect mechanistic anomalies? Suppose that we find a set of arguments $\Pi$ such that the following quantity is low:\footnote{One of the most important and difficult questions faced by this approach is how to find such a $\Pi$. If the space of arguments is parameterized, then we may hope to learn $\Pi$ via gradient descent in parallel with training $M$ itself.}
\[\GG(\EE[x \sim \mathcal{D}]{L(x, M(x))} \mid \Pi).\]
That is, $\Pi$ explains why $M$ attains low average loss on $\mathcal{D}$.\footnote{The idea is that, without any arguments, $\GG$ does not understand anything about the structure of $M$, and so should estimate $M$'s loss as if $M$ were a randomly initialized neural network. (Such a network would incur high loss.) Heuristic arguments that explain $M$'s structure should cause $\GG$'s estimate of $M$'s loss to decrease.} Given an out-of-distribution input $x^*$ such that $L(x^*, M(x^*))$ is once again low, we consider the quantity $\GG(L(x^*, M(x^*)) \mid \Pi)$. This represents a heuristic estimate of $M$'s loss on $x^*$ based only on the reasons provided in $\Pi$: that is, the reasons that explain $M$'s good performance on $\mathcal{D}$. If $\GG(L(x^*, M(x^*)) \mid \Pi)$ is (correctly) low, then the reasons why $M$ performs well on $\mathcal{D}$ also explain why $M$ performs well on $x^*$. By contrast, if $\GG(L(x^*, M(x^*)) \mid \Pi)$ is (incorrectly) high, then $M$ performs well on $x^*$ for a different reason than why $M$ performs well on $\mathcal{D}$. As a result, we flag $x^*$ as a mechanistic anomaly for $M$. See \citet{christiano22mad} for further discussion of this application.

\paragraph{Safe distillation.} Let $f$ (``fast'') and $s$ (``slow'') be two neural networks that were trained on a distribution $\mathcal{D}$ of inputs to complete the same task. Thus, $f$ and $s$ behave similarly on $\mathcal{D}$. Suppose that we trust $s$ to be aligned (e.g.\ we trust $s$ to generalize well off-distribution) and do not similarly trust $f$, but that $s$ is much slower than $f$. Given an out-of-distribution input $x^*$, we would like to estimate $s(x^*)$ without running $s$. We could do this by running $f$ on $x^*$ and hoping that $f$ generalizes well to $x^*$. However, this approach is not very robust. Instead, we can attempt to use the internal activations of $f$ to predict $s(x^*)$.

Concretely, suppose for simplicity that $f$ and $s$ output vectors, and suppose that we find a set of arguments $\Pi$ such that the following quantity is low:
\[\GG(\EE[x \sim \mathcal{D}]{\norm{f(x) - s(x)}_2^2} \mid \Pi).\]
That is, $\Pi$ explains why $f$ and $s$ produce similar outputs on $\mathcal{D}$. Given an out-of-distribution input $x^*$, we consider the quantity
\[\GG(s(x^*) \mid \Pi, \text{computational trace of $f$ on } x^*).\]
This represents a heuristic estimate of $s(x^*)$ given the computations done by $f$ and the argument $\Pi$ for why $f$ and $s$ are similar on $\mathcal{D}$. If the reason why $f$ and $s$ behave similarly on $\mathcal{D}$ also extends to $x^*$, then $\GG$ will correctly estimate $s(x^*)$ to be similar to $f(x^*)$. On the other hand, if the reason why $f$ and $s$ behave similarly on $\mathcal{D}$ does not extend to $x^*$, then $\GG$'s estimate of $s(x^*)$ may be different from $f(x^*)$. This estimate may be more robust to distributional shifts, because it is based on mechanistic reasoning about how $f$ and $s$ work.

\paragraph{Low probability estimation.} Let $M$ be a neural network that was trained on a distribution $\mathcal{D}$. Let $C$ (for ``catastrophe'') be a different neural network that checks the output of $M$ for some rare but highly undesirable behavior: $C(M(x))$ returns $1$ if $M$ exhibits the undesirable behavior on $x$, and $0$ otherwise. We may wish to estimate $\EE[x \sim \mathcal{D}]{C(M(x))}$, and we cannot do so by sampling random inputs $x \sim \mathcal{D}$ because $C$ outputs $1$ very rarely. Suppose that we find a set of arguments $\Pi$ that explains the mechanistic behavior of $M$ and $C$.\footnote{The loss function for learning $\Pi$ is an open question, but would likely involve a notion of ``quality of explantion'': the extent to which $\Pi$ provides an accurate mechanistic explanation of how the activations of $M$ and $C$ change at each layer. See Section~\ref{sec:activaiton_modeling} for more detail.} If this explanation is good enough, then $\GG(\EE[x \sim \mathcal{D}]{C(M(x))} \mid \Pi)$ will be a high-quality estimate of this probability. Additionally, we may use $\GG$ to more efficiently check $M$'s behavior on particular inputs: given an input $x^*$, the quantity
\[\GG(C(M(x^*)) \mid \Pi, \text{computational trace of $M$ on } x^*)\]
represents an estimate of the likelihood that $C(M(x^*)) = 1$ based on the computations done by $M$. This is especially useful if $C$ is slow and running it on every output of $M$ is prohibitively expensive. See \citet{xu24lpe} for further discussion of this application.

\subsection{Activation modeling as a potential path forward} \label{sec:activaiton_modeling}
In light of the challenges discussed in this work, we are now interested in exploring another perspective  on heuristic explanations of neural network behavior. This new perspective views heuristic explanations as sophisticated distributional models of the neural network's internal activations. For example, one could successively fit Gaussian models to each layer of activations, and then use the model to answer questions about the neural network's behavior.\footnote{For example, if we have a model of the layer-$k$ activations of a neural network, then we may be able to detect anomalous behavior by noticing that the neural network's layer-$k$ activations are out of distribution.} Of course, a Gaussian model is not nearly rich enough to explain arbitrary behaviors of neural networks, so a much more sophisticated model class would be necessary.

More broadly, the idea behind this ``activation modeling'' approach is to train a distributional model of a neural network's activations with an eye toward explaining a given property of the neural network. Under this perspective, searching for an argument that explains a given behavior amounts to searching for a model of activations that optimizes some loss function. The loss function would likely have a term for the \emph{consistency} between consecutive-layer activation models, and another term for \emph{explanation quality:} how well the activation model explains the property in question.

In order for this research direction to succeed, we hope to surmount two challenges. First, we would like to find a parameterized class of activation models that is sophisticated enough that it can explain arbitrary neural network behavior. Second, we would like to find a loss function for activation models that, when optimized for, results in an activation model that successfully explains the target behavior. These challenges appear difficult but surmountable, and we are excited to further explore heuristic explanation of neural network behavior from this activation modeling perspective.

\section*{Acknowledgments}
We would like to thank Thomas Read and George Robinson for their comments and suggestions.

\printbibliography

\appendix
\section{Details omitted from Section~\ref{sec:hafnian}} \label{appx:multiacc_omitted}
\subsection{Details omitted from Section~\ref{sec:hardness}}
\noexactmerge*
\begin{proof}
    By Remark~\ref{claim:set_cov}, we have that $\EE[\mathcal{D}_n]{X_{\mathcal{T}} X_{\mathcal{U}}} = \abs{S(\mathcal{T}) \cap S(\mathcal{U})}$. Thus, all we have to show is that computing $\abs{S(\mathcal{T}) \cap S(\mathcal{U})}$ is $\mathsf{\#P}$-hard. We do so by reducing $\#3SAT$ to this problem. That is, we will convert a 3CNF $\varphi$ into a pair of pairing structures $\mathcal{T}, \mathcal{U}$ such that the number of satisfying assignments to $\varphi$ is equal to $\abs{S(\mathcal{T}) \cap S(\mathcal{U})}$.
    
    Let $\varphi$ be a 3CNF on variables $x_1, \dots, x_m$ with $k$ clauses:
    \[\varphi(x_1, \dots, x_m) := \bigwedge_{j = 1}^k (x_{i_{j, 1}} = b_{j, 1} \vee x_{i_{j, 2}} = b_{j, 2} \vee x_{i_{j, 3}} = b_{j, 3}),\]
    where the indices $i_{1, 1}, \dots, i_{k, 3}$ are in $[m]$ and $b_{1, 1}, \dots, b_{k, 3} \in \{0, 1\}$. We will create pairing structures $\mathcal{T}$ and $\mathcal{U}$ on the following index set $K$ with $12k$ elements:\footnote{This is no different from letting $K = [12k]$.}
    \[K := \{(j, s, c_1, c_2) : j \in [k], s \in [3], c_1, c_2 \in \{0, 1\}\}.\]
    We construct $\mathcal{T}$ as follows. For every variable $x_a$, $a \in [m]$, we create a sub-structure:
    \[\mathcal{T}_a := \bigotimes_{j, s: i_{j, s} = a} \parens{\mathcal{T}_{\{(j, s, 0, 0), (j, s, 1, 1)\}} \otimes \mathcal{T}_{\{(j, s, 0, 1), (j, s, 1, 0)\}}} \cup \bigotimes_{j, s: i_{j, s} = a} \parens{\mathcal{T}_{\{(j, s, 0, 0), (j, s, 1, 0)\}} \otimes \mathcal{T}_{\{(j, s, 0, 1), (j, s, 1, 1)\}}}.\]
    Then, we let $\mathcal{T} := \mathcal{T}_1 \otimes \dots \otimes \mathcal{T}_m$.
    
    $\mathcal{T}$ has $2^m$ pairings. In particular, a pairing belongs to $S(\mathcal{T})$ if and only if it takes the form
    \[\bigcup_{j \in [k], s \in [3]} \{\{(j, s, 0, 0)), (j, s, 1, r_{i_{j, s}})\}, \{(j, s, 0, 1), (j, s, 1, 1 - r_{i_{j, s}})\}\}\]
    for some $r_1, \dots, r_m \in \{0, 1\}$. (The structure of each $\mathcal{T}_a$ enforces the same choice of $r_{i_{j, s}}$ for all $j, s$ such that $i_{j, s} = a$.) Later, we will analyze each such pairing as corresponding to the assignment $\{x_1 = r_1, \dots, x_m = r_m\}$.
    
    We construct $\mathcal{U}$ as follows. For every clause $j \in [k]$, we create a sub-structure $\mathcal{U}_j$, which will itself be a union of seven structures. These seven constituent structures will correspond to the seven ways to set the variables in clause $j$ so as to make the clause true
    \[\mathcal{U}_j := \bigcup_{\substack{r_1, r_2, r_3 \in \{0, 1\} \\ r_1 = b_{j, 1} \vee r_2 = b_{j, 2} \vee r_3 = b_{j, 3}}} \parens{\bigotimes_{s = 1}^3 \mathcal{T}_{\{(j, s, 0, 0), (j, s, 1, r_s)\}} \otimes \mathcal{T}_{\{(j, s, 0, 1), (j, s, 1, 1 - r_s)\}}}.\]
    Then, we let $\mathcal{U} := \mathcal{U}_1 \otimes \dots \otimes \mathcal{U}_k$.
    
    We can characterize $S(\mathcal{U})$ as follows: $S(\mathcal{U})$ contains pairings that pair $(j, s, 0, 0)$ with $(j, s, 1, z)$ and $(j, s, 0, 1)$ with $(j, s, 1, 1 - z)$ for some $z \in \{0, 1\}$, but subject to the following ``satisfaction'' condition: for every clause $j$, for at least one $s \in [3]$, $(j, s, 0, 0)$ is paired with $(j, s, 1, b_{j, s})$ while $(j, s, 0, 1)$ is paired with $(j, s, 1, 1 - b_{j, s})$.
    
    We claim that $\abs{S(\mathcal{T}) \cap S(\mathcal{U})}$ is equal to the number of satisfying assignments to $\varphi$. To see this, for every assignment $\{x_1 = r_1, \dots, x_m = r_m\}$, let us consider the pairing
    \[\bigcup_{j \in [k], s \in [3]} \{\{(j, s, 0, 0)), (j, s, 1, r_{i_{j, s}})\}, \{(j, s, 0, 1), (j, s, 1, 1 - r_{i_{j, s}})\}\}.\]
    As discussed, these are precisely the pairings in $\mathcal{T}$. Now, when does such a pairing belong to $\mathcal{U}$? If this assignment satisfies $\varphi$, then for every $j \in [k]$ there exists $s \in [3]$ such that $r_{i_{j, s}} = b_{j, s}$, which means that the pairing belongs to $\mathcal{U}$. Conversely, if the assignment does not satisfy $\varphi$, then there exists $j \in [k]$ such that $r_{i_{j, s}} \neq b_{j, s}$ for all $s \in [3]$, which means that the pairing does not belong to $\mathcal{U}$. Therefore, the pairings that belong to both $\mathcal{T}$ and $\mathcal{U}$ are in correspondence with the satisfying assignments of $\varphi$, as desired.
\end{proof}

\subsection{Details omitted from Section~\ref{sec:approx}}
\begin{algorithm}
    \caption{Approximately accurate merge of pairing structure estimates \label{alg:approx_acc_haf}}
    \begin{algorithmic}[1]
        \Function{NumPairings}{$\mathcal{T}$} \Comment{$\mathcal{T}$ is a pairing structure. Returns $\abs{S(\mathcal{T})}$.}
            \If{$\mathcal{T}$ is a base pairing structure} \Return{1}
            \ElsIf{$\mathcal{T}$ is a product $\mathcal{T}_1 \otimes \mathcal{T}_2$}
                \State \Return{\Call{NumPairings}{$\mathcal{T}_1$} $\cdot$ \Call{NumPairings}{$\mathcal{T}_2$}}
            \ElsIf{$\mathcal{T}$ is a union $\mathcal{T}_1 \cup \mathcal{T}_2$}
                \State \Return{\Call{NumPairings}{$\mathcal{T}_1$} $+$ \Call{NumPairings}{$\mathcal{T}_2$}}
            \EndIf
        \EndFunction
        \Statex
        
        \Function{Sample}{$\mathcal{T}$} \Comment{$\mathcal{T}$ is a pairing structure. Returns a random element of $S(\mathcal{T})$.}
            \If{$\mathcal{T}$ is a base pairing structure $\mathcal{P}_2(K)$} \Return{$K$}
            \ElsIf{$\mathcal{T}$ is a product $\mathcal{T}_1 \otimes \mathcal{T}_2$}
                \State \Return{\Call{Sample}{$\mathcal{T}_1$} $\cup$ \Call{Sample}{$\mathcal{T}_2$}}
            \ElsIf{$\mathcal{T}$ is a union $\mathcal{T}_1 \cup \mathcal{T}_2$}
                \State $r := \text{random number in } [0, 1)$
                \If{$r <$ \Call{NumPairings}{$\mathcal{T}_1$} / \Call{NumPairings}{$\mathcal{T}$}} \Comment{Can be made more efficient by caching recursive outputs of \Call{NumPairings}{} on sub-structures}
                    \State \Return{\Call{Sample}{$\mathcal{T}_1$}}
                \Else{}
                    \Return{\Call{Sample}{$\mathcal{T}_2$}}
                \EndIf
            \EndIf
        \EndFunction
        \Statex
        
        \Function{Split}{$p$, $K$} \Comment{$p$ is a pairing of a superset of $K$. Returns the ``restriction'' of $p$ to $K$, or \texttt{False} if $p$ does not pair up the elements of $K$.}
            \State $p' := \emptyset$
            \For{$\{i, j\} \in p$}
                \If{$\abs{\{i, j\} \cap K} = 1$} \Return{\texttt{False}}
                \ElsIf{$\{i, j\} \subseteq K$} $p' := p' \cup \{\{i, j\}\}$
                \EndIf
            \EndFor
            \Return{$p'$}
        \EndFunction
        \Statex    
        
        \Function{Contains}{$\mathcal{T}$, $p$} \Comment{$\mathcal{T}$ is a pairing structure. $p$ is a pairing of the index set of $\mathcal{T}$. Returns $\texttt{True}$ if $p \in S(\mathcal{T})$.}
            \If{$\mathcal{T}$ is a base pairing structure} \Return{$\texttt{True}$}
            \ElsIf{$\mathcal{T}$ is a product $\mathcal{T}_1 \otimes \mathcal{T}_2$ where $\mathcal{T}_1 = (K_1, \nu_1)$ and $\mathcal{T}_2 = (K_2, \nu_2)$}
                \State $p_1 := $ \Call{Split}{$p$, $K_1$}
                \If{$p_1 = \texttt{False}$} \Return{\texttt{False}}
                \Else{} \Return{\Call{Contains}{$\mathcal{T}_1$, $p_1$} \textbf{ and } \Call{Contains}{$\mathcal{T}_2$, $p \setminus p_1$}}
                \EndIf
            \ElsIf{$\mathcal{T}$ is a union $\mathcal{T}_1 \cup \mathcal{T}_2$}
                \State \Return{\Call{Contains}{$\mathcal{T}_1$, $p$} \textbf{ or } \Call{Contains}{$\mathcal{T}_2$, $p$}}
            \EndIf
        \EndFunction
        
        \algstore{approxalg}
    \end{algorithmic}
\end{algorithm}

\begin{algorithm}
    \begin{algorithmic}[1]
        \algrestore{approxalg}
        
        \Function{Estimator}{$\mathcal{T}_1, \dots, \mathcal{T}_m$, $\delta$, $\varepsilon$} \Comment{$\mathcal{T}_1, \dots, \mathcal{T}_m$ are pairing structures with the same index set, presumed to be ordered by number of pairings: $\abs{S(\mathcal{T}_1)} \ge \dots \ge \abs{S(\mathcal{T}_m)}$. $\delta$ and $\varepsilon$ are the tolerances for approximate multiaccuracy. Returns a function $f: \RR^m \to \RR$ that estimates the hanfian in terms of $X_{\mathcal{T}_1}, \dots, X_{\mathcal{T}_m}$.}
            \State $D := m \times m$ diagonal matrix with entries $\sqrt{\Call{NumPairings}{\mathcal{T}_i}}$
            \State $M := $ empty $m \times m$ matrix \Comment{$M_{i, j}$ counts the number of sampled pairings in $S(\mathcal{T}_j)$ that also belong to $S(\mathcal{T}_i)$}
            \State $s := 0$ \Comment{Tracks the number of samples taken}
            \Do
                \State $s \mathrel{+}= 1$
                \For{$i := 1 \textbf{ to } m$}
                    \State $M_{i, i} \mathrel{+}= 1$ 
                    \For{$j := i + 1 \textbf{ to } m$}
                        \If{\Call{Contains}{$\mathcal{T}_i$, $\Call{Sample}{\mathcal{T}_j}$}} $M_{i, j}, M_{j, i} \mathrel{+}= 1$
                        \EndIf
                    \EndFor
                \EndFor
            \State $\hat{C} := D^{-1}MD/s$ \Comment{Estimated correlation matrix of $X_{\mathcal{T}_1}, \dots, X_{\mathcal{T}_m}$}
            \State $\hat{\sigma}_m :=$ least singular value of $\hat{C}$
        \doWhile{$\hat{\sigma}_m \le 0$ \textbf{ or } $s < \frac{80m^2(m^2 + 3 \ln(2/\delta))}{\varepsilon^2 \hat{\sigma}_m^4}$} \Comment{Loop until enough samples have been taken}
        \State $\pmb{\beta} := D^{-1} \hat{C}^{-1} \diag(D)$ \Comment{Linear regression coefficients}
        \State \Return{\textbf{lambda} $\vect{x}: \pmb{\beta}^\top \vect{x}$} \Comment{Return linear regression estimator with coefficients $\pmb{\beta}$}
        \EndFunction
    \end{algorithmic}
\end{algorithm}

\hafapprox*

\begin{proof}
    Let $\mathcal{T}_1, \dots, \mathcal{T}_m$ be as above. For convenience, we will write $S_i$ in place of $S(\mathcal{T}_i)$ and $X_i$ in place of $X_{\mathcal{T}_i}$ for $i \in [m]$. Assume that $\abs{S_1} \ge \dots \ge \abs{S_m}$ (just as the function \textsc{Estimator} assumes).

Recall that the \emph{correlation} between two random variables $Z_1$ and $Z_2$ is defined as $\frac{\Cov{Z_1, Z_2}}{\sqrt{\Var{Z_1} \Var{Z_2}}}$, and that the \emph{correlation matrix} of variables $Z_1, \dots, Z_m$ contains the correlations between all pairs of variables. From Remark~\ref{claim:set_cov}, we know that the correlation between the predictors $X_i$ and $X_j$ is equal to
    \[\frac{\abs{S_i \cap S_j}}{\sqrt{\abs{S_i} \abs{S_j}}} = \frac{\abs{S_i \cap S_j}}{\abs{S_j}} \cdot \frac{\sqrt{\abs{S_j}}}{\sqrt{\abs{S_i}}}.\]
    Let $C$ be the correlation matrix of $X_1, \dots, X_n$. Let $\hat{C}$ be the estimated correlation matrix at the time that \textsc{Estimator} terminates, and let $\Delta = \hat{C} - C$. We begin by showing that if $\hat{C}$ is a good estimate of $C$, then \textsc{Estimator} returns an approximately accurate estimator of $Y$.

\begin{lemma} \label{lem:if_delta_small}
  If $\norm{\Delta}_2 \le \frac{\varepsilon \sigma_m \hat{\sigma}_m}{2m}$, then the output $f$ of \textsc{Estimator} is $(\varepsilon, \{1, X_{\mathcal{T}_1}, \dots, X_{\mathcal{T}_m}, f\})$-multiaccurate.
\end{lemma}

\begin{proof}
  First, it is clear that $f$ is $1$-accurate, i.e.\ that $f$ has mean zero, since $f$ is always a linear combination of the $X_i$'s which each have mean zero.
  
  Now, assume that $\norm{\Delta}_2 \le \frac{\varepsilon \sigma_m \hat{\sigma}_m}{2m}$. Let $A := \hat{C}^{-1} - C^{-1} = (C + \Delta)^{-1} - C^{-1}$. Then
  \begin{align*}
      \norm{A}_2 &= \max_{\vect{v}: \norm{\vect{v}}_2 = 1} \norm{(C + \Delta)^{-1}\vect{v} - C^{-1}\vect{v}}_2 = \max_{\vect{w}: \norm{(C + \Delta)\vect{w}}_2 = 1} \norm{(I - C^{-1}(C + \Delta)) \vect{w}}_2\\
    &= \max_{\vect{w}: \norm{\hat{C} \vect{w}}_2 = 1} \norm{C^{-1} \Delta \vect{w}}_2 \le \frac{1}{\hat{\sigma}_m} \max_{\vect{v}: \norm{\vect{v}}_2 = 1} \norm{C^{-1} \Delta}_2 \le \frac{1}{\hat{\sigma}_m} \norm{C^{-1}}_2 \norm{\Delta}_2 = \frac{\norm{\Delta}_2}{\sigma_m \hat{\sigma}_m} \le \frac{\varepsilon}{2m}.
  \end{align*}
  
  We first show that $f$ is $(\varepsilon, X_i$)-accurate for each $i \in [m]$. By Remark~\ref{rem:approx_acc_rescale}, this is equivalent to showing that $f$ is $(\varepsilon, \tilde{X}_i)$-accurate, where $\tilde{X}_i := \frac{1}{\sqrt{\abs{S_i}}} X_i$ (i.e.\ $X_i$ scaled to have variance $1$), or in other words, that
  \[\EE{(Y - f) \tilde{X}_i}^2 \le \varepsilon^2 \EE{\tilde{X}_i^2} \EE{f^2} = \varepsilon^2 \EE{f^2}.\]
  Let $\vect{c}_i$ be the $i$-th column of $C$, and let $\vect{d} := \diag(D) = (\sqrt{\abs{S_1}}, \dots, \sqrt{\abs{S_m}})$, where $D$ is defined as in \textsc{Estimator}, and let $\tau := \norm{\vect{d}}_2^2 = \sum_i \abs{S_i}$. Then $\EE{Y \tilde{X}_i} = \sqrt{\abs{S_i}} = \vect{d}^\top C^{-1} \vect{c}_i$. On the other hand, letting $\vect{X} = (X_1, \dots, X_n)$ and $\tilde{\vect{X}} = (\tilde{X}_1, \dots, \tilde{X}_n)$, we have
  \[\EE{f \tilde{X}_i} = \EE{\pmb{\beta}^\top \vect{X} \tilde{X}_i} = \EE{\pmb{\beta}^\top D \tilde{\vect{X}} \tilde{X}_i} = \pmb{\beta}^\top D \vect{c}_i = \vect{d}^\top \hat{C}^{-1} \vect{c}_i.\]
  Letting $A := \hat{C}^{-1} - C^{-1}$, we have
  \[\EE{(Y - f) \tilde{X}_i}^2 = (\vect{d}^\top A \vect{c}_i)^2 \le \norm{\vect{d}}_2^2 \norm{A}_2^2 \norm{\vect{c}_i}_2^2 \le m \tau \norm{A}_2^2 = \frac{\tau \varepsilon^2}{4m}.\]
  By comparison, we have
  \begin{align*}
      \EE{f^2} &= \EE{\pmb{\beta}^\top \vect{X} \vect{X}^\top \pmb{\beta}} = \EE{\pmb{\beta}^\top D \tilde{\vect{X}} \tilde{\vect{X}}^\top D \pmb{\beta}} = \vect{d}^\top \hat{C}^{-1} C \hat{C}^{-1} \vect{d} = \vect{d}^\top (C^{-1} + A) C (C^{-1} + A) \vect{d}\\
      &= \vect{d}^\top C^{-1} \vect{d} + 2\vect{d}^\top A \vect{d} + \vect{d}^\top ACA \vect{d} \ge \abs{S_1} - 2\norm{\vect{d}}_2^2 \norm{A}_2 \ge \frac{\tau}{m} - 2\tau \norm{A}_2 \ge \frac{\tau}{m}(1 - \varepsilon).
  \end{align*}
All steps are self-explanatory, except for the first inequality, which makes the observations that $\vect{d}^\top ACA \vect{d} \ge 0$ and that $\vect{d}^\top C^{-1} \vect{d} \ge \abs{S_1}$. The first fact follows from the fact that $A$ is symmetric and $C$ is PSD. To see the second fact, observe that $\vect{d}^\top C^{-1} \vect{d}$ is precisely $\EE{f_{\text{lin}}^2}$, where $f_{\text{lin}}$ is the optimal linear predictor of $Y$ in terms of $X_1, \dots, X_m$. Observe that
  \[\EE{(f_{\text{lin}} - X_1)X_1} = \EE{(Y - X_1)X_1} - \EE{(Y - f_{\text{lin}}) X_1} = 0 - 0 = 0,\]
  since both $X_1$ and $f_{\text{lin}}$ are $X_1$-accurate estimators of $Y$. In particular, this means that $\EE{f_{\text{lin}}^2} \ge \EE{X_1^2} = \abs{S_1}$.
  
  Therefore, we have that $\EE{(Y - f) \tilde{X}_i}^2 \le \varepsilon^2 \EE{f^2}$ (for $\varepsilon \le 3/4$), as desired.

  Now, we show that $f$ is $(\varepsilon, f)$-accurate, i.e.\ that
  \[\abs{\EE{(Y - f)f}} \le \varepsilon \EE{f^2}.\]
  We have
  \[\EE{Yf} = \pmb{\beta}^\top \EE{\vect{X} Y} = \vect{d}^\top \hat{C}^{-1} \vect{d},\]
  since $\EE{\vect{X} Y} = (\abs{S_1}, \dots, \abs{S_m})$. On the other hand, we showed earlier that $\EE{f^2} = \vect{d}^\top \hat{C}^{-1} C \hat{C}^{-1} \vect{d}$. Thus, we have
  \begin{align*}
      &\abs{\EE{(Y - f)f}} = \abs{\vect{d}^\top \hat{C}^{-1} (I - C \hat{C}^{-1}) \vect{d}} = \abs{\vect{d}^\top (C^{-1} + A) CA \vect{d}} = \abs{\vect{d}^\top (I + AC)A \vect{d}}\\
      &\le \norm{\vect{d}}_2^2 \norm{I + CA}_2 \norm{A}_2 \le \tau (1 + \norm{C}_2 \norm{A}_2) \norm{A}_2 \le \tau (1 + m \norm{A}_2) \norm{A}_2 \le \tau \parens{1 + \frac{\varepsilon}{2}} \frac{\varepsilon}{2m}.
  \end{align*}
  
  Therefore, we have that $\abs{\EE{(Y - f)f}} \le \varepsilon \EE{f^2}$ (for $\varepsilon \le 2/5$), as desired.
\end{proof}

Next, we show that $\hat{C}$ is a good estimate of $C$ with high probability.

\begin{lemma} \label{lem:delta_small_whp}
  $\norm{\Delta}_2 \le \frac{\varepsilon \sigma_m \hat{\sigma}_m}{2m}$ with probability at least $1 - \delta$.
\end{lemma}

\begin{proof}
We will show that $\norm{\Delta}_2 \le \frac{\varepsilon \sigma_m^2}{3m}$ with probability at least $1 - \delta$. To see that this is sufficient, observe that
\[\hat{\sigma}_m = \min_{\norm{\vect{u}}_2 = 1} \norm{\hat{C}\vect{u}}_2 = \min_{\norm{\vect{u}}_2 = 1} \norm{(C + \Delta) \vect{u}}_2 \ge \min_{\norm{\vect{u}}_2 = 1} \norm{C \vect{u}}_2 - \max_{\norm{\vect{u}}_2 = 1} \norm{\Delta \vect{u}}_2 = \sigma_m - \norm{\Delta}_2.\]
\emph{(It is likewise true that $\sigma_m \ge \hat{\sigma}_m - \norm{\Delta}_2$, a fact that will be useful later.)} This means that if $\norm{\Delta}_2 \le \frac{\varepsilon \sigma_m^2}{3m}$, then in fact
\[\hat{\sigma}_m \ge \sigma_m - \norm{\Delta}_2 \ge \sigma_m - \frac{\varepsilon \sigma_m^2}{3m} \ge \frac{2}{3} \sigma_m,\]
since $\sigma_m \le \frac{1}{m} \tr(C) = 1$ and $\frac{\varepsilon}{m} \le 1$.
Thus, if $\norm{\Delta}_2 \le \frac{\varepsilon \sigma_m^2}{3m}$ then in fact $\norm{\Delta}_2 \le \frac{\varepsilon \sigma_m \hat{\sigma}_m}{2m}$.\\

We now wish to show that $\norm{\Delta}_2 \le \frac{\varepsilon \sigma_m^2}{3m}$ with probability at least $1 - \delta$.

Let $N$ be the number of samples taken by \textsc{Estimator}, i.e.\ the value of $s$ when the algorithm terminates. For all $s \le N$, let $\hat{C}_s$ denote the value of $\hat{C}$ after $s$ samples have been taken (and likewise for $M_s$), and let $\Delta_s := \hat{C}_s - C$ (so $\Delta = \Delta_N$). Let $s_0 := \frac{5m^2(m^2 + 3 \ln(3/\delta))}{\varepsilon^2 \sigma_m^4}$. We will prove two facts:

\begin{enumerate}[label=Fact~(\arabic*)]
    \item \label{fact:lowside} The probability that $N < s_0$ is at most $\delta/3$.
    \item \label{fact:highside} The probability that $N \ge s_0$ and $\norm{\Delta}_2 > \frac{\varepsilon \sigma_m^2}{3m}$ is at most $2\delta/3$.
\end{enumerate}
Together, these two facts imply that $\norm{\Delta}_2 \le \frac{\varepsilon \sigma_m^2}{3m}$ with probability at least $1 - \delta$.\\

In order to prove these facts, we need to understand better how $\norm{\Delta}_2$ is distributed. For $1 \le i < j \le m$, we have
    \[(\hat{C}_s)_{i, j} = \frac{(M_s)_{i, j}}{s} \cdot \frac{\sqrt{\abs{S_j}}}{\sqrt{\abs{S_i}}},\]
    where $(M_s)_{i, j}$ is the number of samples of pairings in $S_j$ that were also in $S_i$ (out of $s$ total samples). Thus, $\frac{(M_s)_{i, j}}{s}$ is an unbiased estimator of $\frac{\abs{S_i \cap S_j}}{\abs{S_j}}$: in particular, it is (roughly) normally distributed with mean $\frac{\abs{S_i \cap S_j}}{\abs{S_j}}$ and variance at most $\frac{1}{4s}$. Since $\frac{\sqrt{\abs{S_j}}}{\sqrt{\abs{S_i}}} \le 1$ (as $i \le j$), $(\hat{C}_s)_{i, j}$ is (roughly) normally distributed with mean $C_{i, j}$ and variance at most $\frac{1}{4s}$. (Since $s$ is large, the differences from the normal distribution are negligible, and we will neglect them for ease of presentation.\footnote{To give more detail, the distribution of $\parens{p - \frac{1}{n} B(n, p)}^2$ is extremely similar to the distribution of $\mathcal{N}\parens{0, \sigma^2 = \frac{p(1 - p)}{n}}^2$, and (for sufficiently large $n$) stochastically dominated by $\mathcal{N}\parens{0, \sigma^2 = \frac{p(1 - p)}{1.01n}}^2$, except near zero (i.e.\ for $\abs{x} \le O(1/n^2)$). Our analysis has enough slack to allow for this extra factor of $1.01$.})

Thus, the diagonal entries of $\Delta_s$ are all $0$, while the off-diagonal entries are normally distributed with mean $0$ and variance at most $\frac{1}{4s}$. Further, $\Delta_s$ is symmetric, and the entries on one side of the diagonal are independent. This means that the distribution of $\norm{\Delta_s}_F^2$ (the sum of the squares of the entries of $\Delta_s$) is stochastically dominated by $\frac{1}{2s}$ times the $\chi^2$-distribution with $\frac{m(m - 1)}{2}$ degrees of freedom. Since $\norm{\Delta_s}_2 \le \norm{\Delta_s}_F$ (this is true for all matrices), this is likewise true for the distribution of $\norm{\Delta_s}_2^2$.

We will be using the Laurent-Massart concentration bound on the $\chi^2$-distribution \citep{laurent00chisquared}, which says that for $X \sim \chi^2(k)$, we have
    \[\pr{X \ge k + 2\sqrt{kx} + 2x} \le e^{-x}\]
    for all $x \ge 0$.
In particular, this means that for every positive integer $s$, we have
\[\pr{\norm{\Delta_s}_2^2 \ge \frac{1}{2s} \parens{\frac{m(m - 1)}{2} + 2 \sqrt{\frac{m(m - 1)}{2} \ln(1/\alpha)} + 2 \ln(1/\alpha)}} \le \exp(- \ln(1/\alpha)) = \alpha.\]
This means that
\begin{equation} \label{eq:Delta_bound}
  \pr{\norm{\Delta_s}_2^2 \ge \frac{1}{2s} \parens{m^2 + 3 \ln(1/\alpha)}} \le \alpha,
\end{equation}
because $m(m - 1) \le m^2$ and $2 \sqrt{\frac{m^2}{2} \ln(1/\alpha)} \le \frac{m^2}{2} + \ln(1/\alpha)$.

See below for the remainder of the proofs of \ref{fact:lowside} and \ref{fact:highside}. The proofs are mostly technical applications of equation~\ref{eq:Delta_bound}, although with some nontrivial steps.
\end{proof}

We return to the proof of Theorem~\ref{thm:haf_approx}. It remains only to show that the expected number of samples taken by \textsc{Estimator} is polynomial in $m$, $1/\delta$, $1/\varepsilon$, and $1/\sigma_m$. (Bounding the number of samples suffices, because all other parts of the algorithm are clearly polynomial-time in the sum of the representation sizes of $\mathcal{T}_1, \dots, \mathcal{T}_m$.)

First, note that for $s \ge 16 \cdot \frac{80m^2(m^2 + 3 \ln(2/\delta))}{\varepsilon^2 \sigma_m^4}$, then only way that the stop condition of \textsc{Estimator} is not met after $s$ samples is if $\hat{\sigma}_m \le 0$ or if $\hat{\sigma}_m \ge 2 \sigma_m$. Since $\abs{\sigma_m - \hat{\sigma}_m} \le \norm{\Delta}_2$, this can only happen if $\norm{\Delta}_2 \ge \sigma_m$.

Now, from equation~\ref{eq:Delta_bound}, we know that
\[\pr{\norm{\Delta_s}_2 \ge \sigma_m} = \pr{\norm{\Delta_s}_2^2 \ge \sigma_m^2} \ge \pr{\norm{\Delta_s}_2^2 \ge \frac{1}{2s} \parens{m^2 + 3 \ln \parens{e^{\frac{1}{2} \sigma_m^2 s - m^2}}}} \le e^{m^2 - \frac{1}{2} \sigma_m^2 s}.\]
Thus, the expected number of samples taken by \textsc{Estimator} is at most
\[16 \cdot \frac{80m^2(m^2 + 3 \ln(2/\delta))}{\varepsilon^2 \sigma_m^4} + \sum_{s = 16 \cdot \frac{80m^2(m^2 + 3 \ln(2/\delta))}{\varepsilon^2 \sigma_m^4}}^\infty e^{m^2 - \frac{1}{2} \sigma_m^2 s}.\]
The second term is negligible compared to the first term, and the first term is polynomial in $m$, $1/\delta$, $1/\varepsilon$, and $1/\sigma_m$. This concludes the proof of Theorem~\ref{thm:haf_approx}.
\end{proof}

\begin{proof}[Proof of \ref{fact:lowside}]
First, note that $T \ge \frac{80m^2(m^2 + 3 \ln(3/\delta))}{\varepsilon^2}$, since the least singular value of $\hat{C}_T$ is less than or equal to $1$, and so the stopping condition cannot be met when the number of samples taken is less than $\frac{80m^2(m^2 + 3 \ln(3/\delta))}{\varepsilon^2}$.

We will prove that for every $\frac{80m^2(m^2 + 3 \ln(3/\delta))}{\varepsilon^2} \le s < s_0$, the probability that $s \ge \frac{80m^2(m^2 + 3 \ln(3/\delta))}{\varepsilon^2 \hat{\sigma}_m^4}$ (i.e.\ that the stopping condition is met at $s$ samples) is less than $\frac{\delta}{s^2}$ (where $\hat{\sigma}_m$ here denotes the smallest singular value of $\hat{C}_s$). The union bound is then sufficient to complete the proof, since the sum of $\frac{1}{s^2}$ over all $s \ge \frac{80m^2(m^2 + 3 \ln(3/\delta))}{\varepsilon^2}$ is less than $1/3$.

  By equation~\ref{eq:Delta_bound}, we have that with probability at least $1 - \delta/s^2$, we have
  \begin{equation} \label{eq:Delta_s_small}
      \norm{\Delta_s}_2^2 \le \frac{1}{2s} \parens{m^2 + 3 \ln(s^2/\delta)}.
  \end{equation}
  We also know that $\hat{\sigma}_m \le \sigma_m + \norm{\Delta_s}$. Thus, it suffices to show that if equation~\ref{eq:Delta_s_small} holds, then we have
  \begin{equation} \label{eq:s_small}
      s < \frac{80m^2(m^2 + 3 \ln(3/\delta))}{\varepsilon^2 (\sigma_m + \norm{\Delta_s}_2)^4}.
  \end{equation}
  To prove this, assume (\ref{eq:Delta_s_small}). We start with the observation that $4m^4 < 4s s_0 \sigma_m^4$ and that $36 \ln^2 \parens{s^2/\delta} < 4s s_0 \sigma_m^4$. The former is true because $s \ge 1$ and $s_0 \sigma_m^4 \ge m^4$. The latter is true because
  \[36 \ln^2 \parens{s^2/\delta} = 36(2 \ln s + \ln(1/\delta))^2 \le 72(4 \ln^2 s + \ln^2(1/\delta)) \le 4ss_0 \sigma_m^4,\]
  because $288 \ln^2 s \le s \le s s_0 \sigma_m^4$ and $10 \ln(1/\delta)$ is less than both $s$ and $s_0$.
  
  Therefore, we have that
  \[\parens{m^2 + 3 \ln \parens{s^2/\delta}}^2 \le 4 \max \parens{m^2, 3 \ln \parens{s^2/\delta}} < 4s s_0 \sigma_m^4\]
  and so
  \[16 \parens{\frac{1}{2s \sigma_m^2} \parens{m^2 + 3 \ln \parens{s^2/\delta}}}^2 < 16 \frac{s_0}{s}.\]
  Therefore, we have
  \[\parens{1 + \sqrt{\frac{1}{2s \sigma_m^2} \parens{m^2 + 3 \ln \parens{s^2/\delta}}}}^4 < 16 \frac{s_0}{s}.\]
  This is because for all $x$, if $x \le 1$ then $(1 + x)^4 \le 16 < 16 \frac{s_0}{s}$, and if $x > 1$ then $(1 + x)^4 \le 16x^4$. (We use these facts for $x = \sqrt{\frac{1}{2s \sigma_m^2} \parens{m^2 + 3 \ln \parens{s^2/\delta}}}$.) Now, recall our assumption that equation~\ref{eq:Delta_s_small} holds. This tells us that
  \[\parens{1 + \frac{\norm{\Delta_s}_2}{\sigma_m}}^4 < 16 \frac{s_0}{s},\]
  so
  \[(\sigma_m + \norm{\Delta_s}_2)^4 < 16 \frac{s_0 \sigma_m^4}{s} = \frac{80m^2(m^2 + 3 \ln(3/\delta))}{\varepsilon^2 s}.\]
  equation~\ref{eq:s_small} follows.
\end{proof}

\begin{proof}[Proof of \ref{fact:highside}]
Imagine that, instead of stopping our sampling procedure when \textsc{Estimator} terminates, we sample forever, thus defining an infinite sequence of matrices $\Delta_s$. It suffices to show that with probability at least $1 - 2\delta/3$, the \emph{maximum} value of $\norm{\Delta_s}_2$ over all $s \ge s_0$ is at most $\frac{\varepsilon \sigma_m^2}{3m}$. Since $\norm{\Delta_s}_F \ge \norm{\Delta_s}_2$, it suffices to show that the maximum value of $\norm{\Delta_s}_F$ over all $s \ge s_0$ is at most $\frac{\varepsilon \sigma_m^2}{3m}$.

Consider the thought experiment in which we construct $m \times m$ matrices $\Delta_1', \Delta_2', \dots$ as follows. Corresponding to each $(i, j): i < j$, we keep a running tally of samples from the Bernoulli distribution with parameter $\frac{1}{2}$; then $(\Delta_s')_{i, j} = (\Delta_s')_{j, i}$ is the amount by which the average of the first $s$ samples exceeds $\frac{1}{2}$. Because the entries of $\Delta_s'$ are also (roughly) normally distributed but with at least as wide a spread as the corresponding entries of $\Delta_s$, the distribution of the maximum of $\norm{\Delta_s'}_F$ over all $s \ge s_0$ stochastically dominates the distribution of the maximum of $\norm{\Delta_s}_F$ over all $s \ge s_0$ (up to discrepancies that are exponentially small in $s$).\footnote{Formally, for each $i < j$, one can couple $\{(\Delta_s)_{i, j}\}_{s = s_0}^\infty$ with $\{(\Delta_s')_{i, j}\}_{s = s_0}^\infty$ in such a way that $(\Delta_s)_{i, j}^2 \le (\Delta_s')_{i, j}^2$ for all $s$ (up to exponentially small discrepancies). Now merge these couplings across all $(i, j)$ pairs independently (which is a valid coupling because $\{(\Delta_s)_{i, j}\}_{s = s_0}^\infty$ is independent for every $(i, j)$ pair, and likewise for $\{(\Delta_s')_{i, j}\}_{s = s_0}^\infty$.} Therefore, it suffices to show that the maximum value of $\norm{\Delta_s'}_F$ over all $s \ge s_0$ is at most $\frac{\varepsilon \sigma_m^2}{3m}$.

To see this, we first observe that $\pr{\norm{\Delta'_{s_0}}_F \ge \frac{\varepsilon \sigma_m^2}{3m}} \le \delta/3$. This follows immediately from plugging $s = s_0 = \frac{5m^2(m^2 + 3 \ln(3/\delta))}{\varepsilon^2 \sigma_m^4}$ and $\alpha = \delta/3$ into equation~\ref{eq:Delta_bound} (which holds for $\norm{\Delta'_{s_0}}_F$ just as it does for $\norm{\Delta_s}_F$).

  With this observation in mind, it now suffices to show that
  \begin{equation} \label{eq:exists_s}
      \pr{\exists s \ge s_0: \norm{\Delta_s'}_F \ge \frac{\varepsilon \sigma_m^2}{3m}} \le 2 \pr{\norm{\Delta'_{s_0}}_F \ge \frac{\varepsilon \sigma_m^2}{3m}}.
  \end{equation}
  To see that this is true, it suffices to show that the probability of the right-hand event conditioned on the left-hand event is at least $\frac{1}{2}$. Suppose that such an $s \ge s_0$ exists, and let $T$ be the largest such $s$ (which exists with probability $1$, since $\norm{\Delta'_s}$ tends to zero almost surely). We claim that for all possible $T, \Delta_T'$, the conditional probability
  \[\pr{\norm{\Delta'_{s_0}}_F \ge \frac{\varepsilon \sigma_m^2}{3m} \mid T, \Delta'_T}\]
  is at least $\frac{1}{2}$. To see why, note that $\EE{\Delta_{s_0}' \mid \Delta'_T} = \Delta'_T$: the expected running average of Bernoulli random variables at an earlier point is equal to the running average at a later point. Further, each $(\Delta'_{s_0})_{i, j}$ is (roughly) normally distributed -- and in particular, symmetrically distributed -- around $(\Delta'_T)_{i, j}$. Now, because of the convexity of the square, for any matrix $B$ we have that at least one of $\norm{\Delta'_T + B}_F$ and $\norm{\Delta'_T - B}_F$ is greater than or equal to $\norm{\Delta'_T}_F$. Therefore, the probability that $\norm{\Delta'_{s_0}}_F$ is greater than or equal to $\norm{\Delta'_T}_F$ (which is in turn greater than or equal to $\frac{\varepsilon \sigma_m^2}{3m}$) is greater than or equal to $\frac{1}{2}$, as desired.\footnote{If $T$ is only very slightly larger than $s_0$, the normal approximation for $(\Delta'_{s_0})_{i, j}$ conditioned on $(\Delta'_T)_{i, j}$ may be sufficiently imperfect to matter for this claim; however, replacing $\frac{1}{2}$ with a slightly smaller constant solves this issue.}
\end{proof}
\end{document}